\documentclass[11pt]{article}
\usepackage[utf8]{inputenc} 
\usepackage[T1]{fontenc} 
\usepackage{hyperref} 
\usepackage{url}         
\usepackage{booktabs}       
\usepackage{amsfonts}    
\usepackage{geometry}
\usepackage{nicefrac}       
\usepackage{microtype}      
\usepackage{xcolor}         

\usepackage{defs}

\title{Decentralized, Communication- and Coordination-free Learning in Structured Matching Markets}

\author{%
  Chinmay Maheshwari,
   Eric Mazumdar, and
   Shankar Sastry \thanks{C. Maheshwari(chinmay\_maheshwari@berkeley.edu) and S. Sastry (shankar\_sastry@berkeley.edu) are with EECS department at University of California Berkeley. E. Mazumdar (mazumdar@caltech.edu) is with CMS And Economics department at Caltech.}
}

\date{}
\begin{document}

\maketitle

\begin{abstract}
 We study the problem of online learning in competitive settings in the context of two-sided matching markets. In particular, one side of the market, the agents, must learn about their preferences over the other side, the firms, through repeated interaction while competing with other agents for successful matches. We propose a class of decentralized, communication- and coordination-free algorithms that agents can use to reach to their stable match in structured matching markets. In contrast to prior works, the proposed algorithms make decisions based solely on an agent's own history of play and requires no foreknowledge of the firms' preferences. Our algorithms are constructed by splitting up the statistical problem of learning one's preferences, from noisy observations, from the problem of competing for firms. We show that under realistic structural assumptions on the underlying preferences of the agents and firms, the proposed algorithms incur a regret which grows at most logarithmically in the time horizon. Our results show that, in the case of matching markets, competition need not drastically affect the performance of decentralized, communication and coordination free online learning algorithms.
\end{abstract}

\section{Introduction}







Online decision-making under uncertainty is one of the central problems in modern machine learning, reflecting the need for efficient and high performing algorithms for real-time learning in real-world settings. Despite being such a well-researched area, there is a broad lack of understanding of how to deploy online learning algorithms into settings in which they must compete with each other for resources or information. Indeed, while classic problems of online learning deal with trading off the exploration of possible choices and the exploitation of current knowledge (i.e., the exploration-exploitation tradeoff ~\cite{lattimore2020bandit,slivkins2019introduction}), the addition of competition adds a new axis upon which algorithms must operate~\cite{mansour2017competing,aridor2020competing}--- namely that of competing (perhaps unsuccessfully) for highly desired outcomes or settling for less desired (but also less competitive) outcomes. Broadly, speaking, the dominant approach to dealing with competition in machine learning remains to treat opponents as adversarial\cite{cesa2006prediction}, despite a long literature in economics and game theory~\cite{littman1994markov,fudenberg1998theory}
showing how agents who understand the competitive structure of problems can sometimes vastly outperform solutions based upon worst-case modeling. 

In this paper, we address the problem of online learning in competitive settings in the context of \emph{two-sided matching markets}. Two-sided matching markets \emph{match} users on one side of the market to those on the other to facilitate the exchange of goods or services.
In such settings, each user on one side of the market has an inherent preference ordering for the users on the other side of the market. Since each user seeks to find their most desired match, this results in a game in which a natural equilibrium is that of a \emph{stable matching} wherein no two users would prefer switching from their current match to each other given their preferences. In seminal work, \cite{gale1962college} proposed a simple and effective algorithm--- the \emph{Deferred Acceptance (DA) Algorithm}--- that users on one side of the market can implement to find such a solution when every user knows their own preferences. The algorithm has been widely used in examples ranging from kidney exchanges to medical resident matching where preferences can be assigned or reported to a central authority which does the matching. However, recent years have seen the emergence of a new form of \emph{online} matching markets like  online labor markets (e.g. TaskRabbit, Upwork), online dating markets (e.g. Tinder, Match.com), online crowdsourcing platforms (e.g. Amazon mechanical turk) where the users do not know their preferences apriori, and can repeatedly interact with the market to improve their match quality.

Motivated by these applications we consider a generalization of the problem studied in the seminal paper \cite{gale1962college} wherein one side of the market--- the agents--- do not know their own preferences, but are able to interact repeatedly with the market. In particular, we analyze a repeated game in which, at each round, agents can request to match with a user or firm on the other side of the market. If, at a given round, multiple agents request the same firm, the firm--- assumed to be a myopic utility maximizer--- accepts the request of its most preferred agent (who receives a noisy measurement of their utility of the match from which they can learn their preferences) and rejects the others (who receive no information about their preferences). This setup serves has been studied in a line of recent works on online matching markets~\cite{liu2020competing,liu2021bandit,sankararaman2021dominate,basu2021beyond}. 

Successful algorithms for this framework must simultaneously solve a statistical learning problem (that of learning about their own preferences) and a competitive problem (ensuring that agents get their most desired match despite the presence of other self-interested agents in the market). Previous works for addressing this problem propose algorithms that are centralized~\cite{liu2020competing} (whereby agents send their current beliefs over their preferences to a central platform which does the matching), require coordination between agents (i.e., a choreographed set of strategies to minimize rejections)~\cite{sankararaman2021dominate,basu2021beyond}, or require agents to fully observe the market outcomes of other agents \cite{liu2021bandit}. In contrast, the DA algorithm--- which we take to be the full-information benchmark to which we compare algorithms--- is (i) fully decentralized, (ii) coordination-free, and (iii) requires agents to make decisions only based upon their own history of rejections and successful matchings. Designing learning  algorithms that operate under conditions (i)-(iii) ensures scalability and privacy in large-scale systems where it is unrealistic to assume that agents can keep track of all other agents' matchings. Thus in this work we focus on the question:

\begin{quote}
   \textbf{Does there exist decentralized and coordination-free algorithms that are based only on local history of interactions which provably converges to stable matching?} 
\end{quote}


\paragraph{Contributions.}

In this work we design algorithms for learning while matching in a class of structured matching markets known as \(\alpha-\)reducible matching markets. This condition ensures that there exists an unique stable matching and encompasses many realistic preference structures including serial dictatorship and no crossing conditions \cite{Clark}. We show that the proposed algorithms incur a stable regret with respect to the unique stable matching that grows at most logarithmically in the time horizon.  
The particular contributions of this paper are: 
\begin{enumerate}[leftmargin=*]
    \item We present a general framework for the construction of decentralized, communication, and coordination-free algorithms for learning while matching. In particular, we combine index-based stochastic bandit algorithms (in particular the Upper Confidence Bounds algorithm and Thompson Sampling)                     \cite{auer2002using}, \cite{lattimore2020bandit,slivkins2019introduction} for solving the statistical problem of learning an agent's preferences with a path-length adversarial bandit algorithm \cite{bubeck2019improved,wei2018more} for dealing with the competitive problem. The resulting algorithms make are fully decentralized, and communication and coordination-free since they make use of only an agent's history of collisions, matches, and rewards to choose which firm to request at a given time.
    Furthermore the algorithms are ``any-time'' algorithms, in that they do not require knowledge of time horizon and do not require any specific parameters of the bandit instance beyond the sub-gaussian parameter of the noise. 
    \item We show that when the agents' and firms' preferences satisfy the  \(\alpha-\)reducibility condition and \emph{every} agent uses the algorithm, the regret  accumulated by any agent \(a\) against the stable match is $O\left(\frac{C_{a} |\actSet||\firmSet|log(T)}{\Delta^2}\right)$ where \(\actSet\) is the set of agents, \(\firmSet\) is the set of firms, $\Delta$ is the minimum sub-optimality gap of any agent in the market, and \(C_a\) is a constant that depends on the \(\alpha-\)reducible structure of the market.
\end{enumerate}

\paragraph{Organization}
The paper is organized as follows: In Section \ref{sec: RelatedWorks} we discuss and compare the prior literature related to the focus of this paper. In Section \ref{sec: Setting} we introduce the general problem setup, introduce matching markets and discuss the structural assumptions on the preferences of agents and firms. In Section \ref{sec: Algorithm} we present the algorithmic design paradigm along with a specific algorithm, based on Upper Confidence Bound. In Section \ref{ssec: RegretBound} we show that the algorithm incurs $O(\log(T))$ regret along with a brief sketch of the proof. In Section \ref{sec: Numerics} we study the performance of the algorithm in simulation. We conclude the paper in Section \ref{sec: Conclusion} and also provide some future research directions. The proofs of our results  are  relegated to the Appendix. Moreover, we introduce another important variant of algorithm based on  Thompson Sampling with similar results in the Appendix.

\section{Related works}\label{sec: RelatedWorks}

Sequential decision-making under uncertainty has been extensively studied in machine learning under the guise of multi-armed bandit (MAB) problems. In general, MAB problems can be split into two distinct flavors, which differ in the type of feedback agents receive. Crucially, in both problems the key is trading off exploration of actions and exploiting ones current knowledge.

In the first class of MAB problems, the stochastic MAB, playing an action results in an unbiased estimate of the utility of playing that action. Solutions to the problem can be split among two dominant algorithmic paradigms. The first, based on principle of optimism in the face of uncertainty encompasses the well known upper confidence bounds (UCB) algorithm~\cite{lattimore2020bandit,lai1985asymptotically} and its variants, while the second, based on Thompson sampling takes a Bayesian approach~\cite{russoTS,Thompson} Each of these approaches are known to have optimal performance measured in terms of  \emph{regret:} the expected cumulative utility generated from the algorithm's chosen actions compared to the expected utility that could have been generated from always choosing the best possible action (i.e., the best action that one would choose with full information)~\cite{lattimore2020bandit,agrawal2012analysis}. In particular, these algorithms are known to incur \emph{logarithmic} regret, i.e., regret that grows at most logarithmically over time--- which is known to be optimal for this class of problems up to constant factors. In our paper we present an algorithmic framework for learning in matching markets that works with either class of algorithm, and further incurs logarithmic regret \emph{even} while dealing with competition. 

The second class of multi-armed bandit problems, coming from the literature on learning in games, seeks algorithms that can perform against arbitrary feedback sequences \cite{cesa2006prediction}. Solutions to this class of problems, known as adversarial bandit algorithms, are an active research topic. While it is well known that using simple strategies like multiplicative weights can guarantee regret against the best fixed action in hindsight on the order of $\sqrt{T}$ against worst-case adversaries \cite{cesa2006prediction}, designing algorithms that can improve upon this when adversaries are \emph{not} worst case remains an open research problem. In this paper we leverage advances on the development of \emph{path-length} adversarial regret algorithms that address this problem and guarantee regret that directly depends on the amount of variation an adversary presents~\cite{bubeck2019improved, wei2018more}. 

We briefly remark that there exists several lines of research on multi-agent bandits. One of them is on multi-agent bandits with collisions (with applications primarily in the area of spectrum sharing in wireless networks\cite{liu2010distributed,kalathil2014decentralized,rosenski2016multi,lugosi2021multiplayer,MAMAB}). In such models the arms do not have preferences and if more than one agents collide at any arm then no one receives any utility or attains maximum possible loss. However, these models differ from us since we consider that both sides of markets have preference over one another and when there is a collision only one agents gets matched. Another line of research deals with the problem of collaboratively learning an instance of multi-armed bandit \cite{buccapatnam2015information,chakraborty2017coordinated,sankararaman2019social} where agents can communicate. Note that in these settings there is no competition that is more than one agents apply at same arm at same time. 

The particular intersection of MABs and two-sided matching markets that we analyze has seen a flurry of recent works \cite{liu2020competing,liu2021bandit,basu2021beyond,sankararaman2021dominate}. To the best of our knowledge, \cite{das2005two}, presented the first numerical study on effectively using MAB algorithms to learn preferences in matching markets.  However, it was only recently that \cite{liu2020competing} rigorously formulated the bandit learning problem in the matching markets, and generalized the notion of \emph{regret} from the MAB literature to matching markets in terms of \emph{stable regret}--- i.e., the expected cumulative utility benchmarked against the expected cumulative reward that would have been received if everyone in the market requested their match in a certain stable match\footnote{Note that the stable matching need not be unique in general. Thus the stable regret has to be always specified with respect to which stable matching is being used. Typically, in literature two main stable matchings are considered namely \emph{agent optimal stable matching} and \emph{firm optimal stable matching}.}.  Moreover, they proposed a \emph{centralized} UCB-based  algorithm that facilitates the matching between agents and firms given each agents' current beliefs over their preferences and history of play, while ensuring that \(\bigo(|\actSet||\firmSet|\log(T))\) regret for a UCB based algorithm, where $\actSet$ is the set of agents, $\firmSet$ is the set of firms, and $T$ is the time horizon of the problem. 
In follow up work \cite{liu2021bandit} proposed a \emph{decentralized} bandit learning algorithm based on UCB that allows each user to take its decision in a decentralized manner and still ``converge'' to stable matching while incurring \(O(\exp(|\firmSet|^4)\log^2(T))\) regret. More recently \cite{kong2022thompson} proposed a thompson sampling based variant of \cite{liu2021bandit}. However, these algorithms requires the knowledge of outcomes at other firms at every round, leaving algorithms that are based solely on agents' own history of play as an open problem.
Concurrently, \cite{sankararaman2021dominate} proposed an algorithm that works in phases and makes use of communication between agents to coordinate agents' actions. Under this information structure the algorithm achieves \(\bigo\lr{|\firmSet|^2|\actSet|^2\log(T)}\) regret. Moreover their guarantees require that firms have homogeneous preference over the agents (also referred as \emph{serial dictatorship}). Follow-up work, \cite{basu2021beyond} improved the regret for serial dictatorship to \(\bigo\lr{|\firmSet||\actSet|\log(T)}\) by proposing a new algorithm. Additionally, they also showed that if the assumption of serial dictatorship is relaxed to a weaker structural condition then they obtain \(O(poly(|\actSet|,|\firmSet|)\log(T))\) regret. Even though the proposed algorithm in \cite{basu2021beyond} has decentralization it is a phase based algorithm,  the agents act according to a coordinated protocol at some rounds. In this paper we propose a simple, decentralized, communication and coordination free algorithm in which agents make use of their own local information to learn while matching. Unlike previous works \cite{liu2020competing,liu2021bandit,sankararaman2021dominate,basu2021beyond} where the algorithms are constructed using a UCB subroutine, we also show that our algorithmic design paradigm can be also seamlessly extended to Thompson sampling variant. 

We would also like to remark about another line of research at the intersection of multiarmed bandits and matching markets  \cite{jagadeesan2021learning}, \cite{johari2016matching} ,\cite{cen2021regret} which consider the problem of learning preferences from the perspective of a platform.

\section{Setting}\label{sec: Setting}
\label{ssec: Setup}
We define a two-sided market \(\market\) as collection of agents \(\actSet\) and firms \(\firmSet\). In the setting under consideration, we assume that every agent $a \in \actSet$ has \emph{unknown} preferences over firms $f \in \firmSet$ which are captured by utilities $\utilityAgent_a(f) \in \mb{R}$. Moreover, no two firms give the same utility to a given agent, i.e. $\utilityAgent_a(f)\ne \utilityAgent_a(f')$ if 
\(f\neq f'\).
We assume that every agent seeks to be matched to their most preferred firm, and that firms have preferences over all the agents which are also captured by utilities  $\utilityFirm_f(a)$ for each $a$ and $f$ such that no two agents give same utility to firms i.e.
\(\utilityFirm_f(a)\neq \utilityFirm_f(a')\) . Importantly, we assume that firms know their own preference orderings over agents and that there are more firms than agents, i.e.
$|\actSet|\leq |\firmSet|$.
The interaction between agents and firms happens as follows: In each time step $t=1,\ldots,\horizon$ every agent $a\in \actSet$ independently \emph{\queries} a firm $\chosenFirm_a(t)\in\firmSet$. As the agents \query\ independently, it is possible that more than one agent \queries\ the same firm \(f\).  For $f \in \firmSet$, let $\agentsPull_f(t) \defas \{a\in\actSet: \chosenFirm_a(t)=f\}$ denote the set of agents that \query \ firm $f$ at time step $t$. At each time step $t$, we assume that the firm $f$ accepts the \query\ of their most preferred agent in $\agentsPull_f(t)$ denoted by $\chosenAgent_f(t)\defas \argmax_{a \in \agentsPull_f(t)} \utilityFirm_f(a)$, and rejects the \query\ of all other agents. That agent \(\chosenAgent_f(t)\) is said to be the agent who got \emph{matched} with firm \(f\) at time \(t\). 
Moreover every matched agent receives a noisy measurement of their utility, denoted $\reward_{\textbf{a},f}$ such that
\begin{align}\label{eq: RewardModel}
\reward_{\textbf{a},f}= \utilityAgent_{\textbf{a}}(f)+\noise_{\textbf{a},f},
\end{align}
where $\noise_{\textbf{a},f}$ is a zero-mean, one-sub-Gaussian random variable. 
Meanwhile, all the agents that are rejected are said to have \emph{collided} on firm $f$, for which they receive no utility i.e. $\reward_{a,f}(t)=0$. 

We restrict that agents \emph{only} receive the following information at any time step \(t\):
\begin{enumerate}[leftmargin=*]
\item \label{enum: info1} \(\isMatched_{a}(t) = \one\lr{\text{\(a\) is matched to \(f_a(t)\)}}.\) which captures if agent \(a\) gets matched at time \(t\)
\item  \label{enum: info2} if they get matched, the noisy measurement of their utility, $U_{a,f}(t)$.
\end{enumerate}
\begin{remark}
We note that in this setup an agent does not know \emph{anything} about how other agents are performing in the market. Agents do \emph{not} observe who gets successfully matched on firms that they have requested and do not observe who they have collided with. We remark that this is \emph{the same} information structure as that assumed by the DA algorithm and is the key assumption that differentiates our work from prior work on this problem~\cite{liu2020competing,liu2021bandit,basu2021beyond,sankararaman2021dominate}.
\end{remark}

In the following subsection, we recall some important results from matching market literature that are crucial to further exposition. 
\subsection{Preliminaries on matching markets}\label{ssec: PrelimonMatching}
To analyze the matching market defined in the previous section we recall key concepts from the literature on matching markets. 
A matching $\matching: \actSet\ra \firmSet$ is an injective function such that $\matching(a)=f$ denotes that $a$ and firm $f$ are matched. 
We call a matching \emph{unstable} if there is an agent-firm tuple $(a,f)\in\actSet\times \firmSet$ such that $\utilityAgent_a(\matching(a))<\utilityAgent_a(f)$ and $\utilityFirm_f(a)>\utilityFirm_f(\matching^{-1}(f))$. In words, there is a pair $(a,f)$ who both prefer each other over their current match, such pair is called a \emph{blocking pair}. A matching is \emph{stable} if it is not unstable.
It is usually the case that a market may have multiple stable matchings. However, for the purpose of this paper we focus on markets which are \(\alpha-\)reducible, first introduced in \cite{Alcalde} and further analyzed in \cite{Clark}, that ensures there is a unique stable matching. Before formally describing this property we introduce the notion of a submarket and fixed pair.

A sub-market of $\market$ is a market $\market'$ such that $\market'=\actSet' \cup  \firmSet'$  where $\actSet'\subseteq \actSet$, $\firmSet' \subseteq \firmSet$, and $|\actSet'|\le |\firmSet'|$. Meanwhile, a pair $(a,f)\in \actSet \times\firmSet$ is a \emph{fixed pair} of market \(\market\) if $\utilityAgent_a(f)\geq\utilityAgent_a(f')$ for all $f' \in \firmSet$ and $\utilityFirm_f(a) \geq \utilityFirm_f(a')$ for all $a' \in \actSet$.
In words, a fixed pair is any agent-firm pair that prefer each other over any other options in the market.  We now define the notion of \(\alpha-\)reducibility.
\begin{definition}[$\alpha$-reducibility]
A market $\mc{M}=\actSet \cup  \firmSet$ is $\alpha$-reducible if every sub-market of \(\market\) has a fixed pair.\label{def:alpharep}
\end{definition}
{The notion of $\alpha$-reducibility is weaker than the \emph{no crossing condition} and serial dictatorship \cite{Clark}}. These conditions have been introduced in the effort to characterize the existence and uniqueness of a stable matching. In \cite{Clark} the authors show that every sub-market of of $\market$ has a unique stable matching if  $\market$ is $\alpha$-reducible.

The preceding property of \(\alpha-\)reducible markets will be crucial to obtain regret guarantees for the proposed algorithm in this paper. 
 Thus, we assume that $\market$ is $\alpha$-reducible.

\begin{remark}\label{rem: MarketDecomp}
An important property of  \(\alpha-\)reducibility assumption that is central to the subsequent analysis is that it allows us to partition the market into various sub-markets by sequentially eliminating fixed pairs. More formally, lets define 
\(\subActSet_0=\subFirmSet_0= \varnothing\) and \(\market_0 = \market\). 
Now for \(i\geq 1\) lets define inductively
\(\subActSet_i\subseteq\actSet\backslash\{\cup_{j=1}^{i}\subActSet_{j-1}\}, \subFirmSet_i\subseteq \firmSet\backslash \{\cup_{j=1}^{i}\subFirmSet_{j-1}\}\) 
be the set of agents and set of firms that constitute fixed pair in market \(\market_{i-1}\).  That is, for every agent \(a\in\subActSet_i\) there exists a unique
\(f\in\subFirmSet_i\) such that \((a,f)\) is a fixed pair of market \(\market_{i-1}\).  
The iteration evolves as
\(\market_i \defas \{\actSet\backslash\{\cup_{j=0}^{i}\subActSet_{j}\}\}\cup \{\firmSet\backslash \{\cup_{j=0}^{i}\subFirmSet_{j}\}\}\).   
 Let \(\numMarket\) be the total number of such sub-markets \(\{\market_i\}\). 
Moreover such {decomposition} of market is unique.
\end{remark}

For any agent \(a\in \actSet\) we denote by \(\stableArm_a\) its match in the unique stable matching. Furthermore, let \(\superArm_a \defas \lb{f\in\firmSet: \utilityAgent_a(f)> \utilityAgent_a(\stableArm_a)}\) be the set of firms that agent \(a\) prefers over its stable match. We call such firms \emph{super-optimal} firms for \(a\). Similarly, let \(\subArm_a \defas \lb{f\in\firmSet: \utilityAgent_a(f)< \utilityAgent_a(\stableArm_a)}\) be the set of firms which are less preferred than the stable match by agent \(a\). We call such firms \emph{sub-optimal} firms for \(a\). 
Note that we have following lemma which states a crucial property of super-optimal firms for \(\alpha-\)reducible markets.
\begin{lemma}\label{lem: SuperOptimalArm}
For any \(i\in[\numMarket]\) and agent \(a\in\subActSet_i\) the set of super-optimal firms are contained in \(\cup_{j=1}^{i-1}\subFirmSet_j\).
\end{lemma}

An immediate conclusion of Lemma \ref{lem: SuperOptimalArm} is that it creates a hierarchy in the market. That is,
an agent \(a\in \subActSet_i\), for some \(i\in[\numMarket]\), is in a sense ``higher ranked''
than a agent \(a'\in \subActSet_j\) for \(j>i\) as the former's stable match can be super-optimal for the latter. This sort of hierarchy naturally manifests itself in the learning process where learning of agent \(a\) creates \emph{externality} for agent \(a'\). 

For ease of reference, all key notations used in paper are presented in a table in the Appendix.

\section{Description of the Algorithm}
\label{sec: Algorithm}
In this section we present a novel algorithm design principle for agents to learn about the preferences while ensuring that they perform competitively against the match that they could have achieved if they knew their preferences and used the DA algorithm. Throughout this section, we assume that every agent \(a\in \actSet\) uses these algorithms in order to decide which firm to choose at time any time \(t\). The proposed  algorithms---by design--- make use of only the feedback information outlined in (\ref{enum: info1})-(\ref{enum: info2}) in Section 1, and have no implicit or explicit communication and coordination strategies like e.g., phase based strategies with coordinated actions~\cite{basu2021beyond} or partial observation of actions of other agents~\cite{liu2021bandit} etc. Thus, the algorithms operate in the same regime as the DA algorithm, but without the assumption that agents know their preferences. Key to our approach, is the blending stochastic bandit (SB) algorithms with an adversarial bandit (AB) algorithms. In the subsequent exposition we will formally describe our approach and show its desirable properties in terms of regret and convergence. 

Before doing so, however, we comment on the difficulties of the problem at hand, and what makes the analysis of these algorithms highly non-trivial. The key challenge in designing algorithms for matching while learning is understanding when to stop requesting \emph{super-optimal} firms  (i.e. firms that they prefer more than their stable match) without any foreknowledge of the market structure. The crux of this problem is having an agent learn that certain firms are unattainable due to competition despite the non-stationarity in the environment stemming from fact that other agents are learning simultaneously and not knowing who they collide with and who is successfully getting matched at each round. Furthermore, due to a lack of communication or coordination, agents cannot learn about which firms are super-optimal without risking many collisions. 

 A sketch of the algorithm is described in words in Algorithm \ref{alg:HighLevelFinal}, and the exact algorithm for the setting in which agents use the UCB algorithm as a subroutine is presented in Algorithm \ref{alg:prunedUCBFinal}.
\begin{algorithm}
Each agent \(a \in \actSet\) at every time \(t\in[T]\):
\begin{enumerate}[leftmargin=*]
    \item Keeps a ordering of firms as per an index-based stochastic bandit subroutine
\item Agent \(a\) goes over the firms as per the ordering one by one 
\item  Using an adversarial bandit subroutine  decides whether to \emph{request} the firm or \newline to \emph{prune it}
\begin{enumerate}[leftmargin=*]
    \item If a firm is requested then agent either gets matched or gets collided
    \item If pruned then then the agent moves to next firm as per the ordering  
\end{enumerate}
\item Updates the stochastic and adversarial bandit subroutine  based on the feedback \newline received
\end{enumerate}
\caption{\textsf{High-level algorithmic description}}
    \label{alg:HighLevelFinal}
\end{algorithm}
 As per Algorithm \ref{alg:prunedUCBFinal}, each agent is equipped with a stochastic bandit (SB) subroutine.  At every time step \(t\in[T]\), the SB subroutine  of every agent \(a\) maintains ordering of firms in decreasing order of preferences according to an index (e.g. UCB). We denote this index of firm \(f\) as maintained by agent \(a\) as \(\UCB_{a,f}(t)\). Next, at that time step, every agent \emph{considers} each firm one by one in decreasing order of \(\UCB_{a,f}(t)\). For any firm \(f\) considered by agent \(a\) at time \(t\), the agent makes a decision to either \emph{request} \(f\) or to \emph{prune}\footnote{Note that by pruning here we do not mean permanent pruning, it is used to describe that a particular firm is not consider at that time step} it (that is, to reject that firm). In particular, agent \(a\) requests firm \(f\) with probability \(\pullProb_{a,f}(t)\). Let \(\pullInstant_{a,f}(t)\sim\text{Bernoulli}(\pullProb_{a,f}(t))\). 
If a firm is pruned (i.e. \(\pullInstant_{a,f}(t)=0\)) then the next best firm from the sorted list is chosen and the process continues until a firm is requested (i.e. \(\pullInstant_{a,f}(t)=1\)). However, if all of the firms are pruned then at that time instant the agent simply requests the firm \(\arg\max_{f}\UCB_{a,f}(t)\). 
Once an agent decides which firm to request, it obtains a noisy utility if it gets successfully matched. This feedback is used by the agent to update its UCB-index. Based on whether an agent \(a\) decides to prune or request a particular firm \(f\), it updates \(\pullProb_{a,f}\) using an AB subroutine. The details about this are stated in Section\footnote{The corresponding algorithmic subroutine \pullModule \ is presented in the Appendix.}  \ref{ssec: AdvBanditAlg}
We note that all firms are not considered by agent \(a\) at every time \(t\). Once an agent decides to request a firm \(f\), it does not consider firms in the set \(\{f'\in\firmSet: \indexSet_{a,f'}(t)<\indexSet_{a,f}(t)\}\). Formally, for any agent-firm tuple \((a,f)\in \actSet\times\firmSet\) let the event that the agent \(a\) {\selects}\ the firm \(f\) at time \(t\), to decide whether to request it or prune it,  be denoted by \(\actQuery_{a,f}(t) = \one\lr{ \pullInstant_{a,f'}(t) = 0, \quad \forall \ f': \indexSet_{a,f}(t) \leq \indexSet_{a,f'}(t)}\). If a firm \(f\) is considered by agent \(a\) then the event when agent \(a\) requests \(f\) is denoted by \(\actChoose_{a,f}(t) = \one\lr{\pullInstant_{a,f}(t)= 1, \actQuery_{a,f}(t) = 1}\).

\begin{algorithm}
\SetAlgoLined
\LinesNumbered
\SetKwInOut{Initialize}{Initialize}
    \Initialize{ $\mean_{a,f}=0,\numMatches_{a,f}=0,\pullProb_{a,f} = 0.5, \auxProb_{a,f} = 0.5, \lossPull_{a,f} = 0,\ \ \forall a\in\actSet, f \in \firmSet  $}
    \For{$t=1, \ldots , \horizon$}{
        \For{$f \in \firmSet$}{
             Set $\UCB_{a,f} = \mean_{a,f}+\sqrt{\frac{2\log(1+(\totalMatch_a+1)\log^2(\totalMatch_a+1))}{\numMatches_{a,f}}}$, where \(\totalMatch_a=\sum_{f\in\firmSet}\numMatches_{a,f}\)
            }
             Set $\argUCB_a$ = \textsf{ArgDescendingSort}($\{\UCB_{a,f}\}_{f\in\firmSet}$) 
            and \(i = 1\)\\
             \While{\(i\leq |\firmSet|\)}
             {
             Set \(f = \argUCB_a^{[i]}\) and
             sample \(\pullInstant_{a,f}\sim \textsf{Bernoulli}(\pullProb_{a ,f})\)\\
            
            \If{\(\pullInstant_{a,f}=0\)}{
            Update \((\auxProb_{a,f},\pullProb_{a,f},\lossPull_{a,f}) \la  \pullModule(\pullInstant_{a,f},\auxProb_{a,f},\pullProb_{a,f},\lossPull_{a,f},\isMatched_{a})\)\\
            }
            \If{\(\pullInstant_{a,f}=1\)}{
            Request firm \(f\) and
            receive $(\reward_{a},\isMatched_{a})$\\
            Update \(\mean_{a,f} \la \isMatched_{a}\frac{\mean_{a,f}\numMatches_{a,f}+\reward_{a}}{\numMatches_{a,f}+1} + (1-\isMatched_{a})\mean_{a,f} \), \ \ $\numMatches_{a,f}\la \numMatches_{a,f}+\isMatched_{a}$,  \\
        Update \((\auxProb_{a,f},\pullProb_{a,f},\lossPull_{a,f}) \la  \pullModule(\pullInstant_{a,f},\auxProb_{a,f},\pullProb_{a,f},\lossPull_{a,f},\isMatched_{a})\)\\
    \textsf{break while};
            }
            
            \(i \la i + 1\)
            }
            \If{\(i=|\firmSet|+1\)}
            {
            Request firm \(\argUCB_a^{[1]}\) and
            receive $(\reward_{a},\isMatched_{a})$\\
            Update \(\mean_{a,f} \la \isMatched_{a}\frac{\mean_{a,f}\numMatches_{a,f}+\reward_{a}}{\numMatches_{a,f}+1} + (1-\isMatched_{a})\mean_{a,f} \),\ \ $\numMatches_{a,f}\la \numMatches_{a,f}+\isMatched_{a}$
            }
   }
   \caption{\textsf{UCB based Decentralized Matching Algorithm (UCB-DMA)}}
    \label{alg:prunedUCBFinal}
\end{algorithm}

 In the Section \ref{ssec: StoBanditAlg} we describe the UCB computation method for the SB subroutine. 
 Finally, in Section \ref{ssec: AdvBanditAlg}, we illustrate how the matchings and collisions are used to update the probability \(\pullProb_{a,f}(t)\) as per an AB subroutine.

\subsection{Stochastic Bandit Subroutine }\label{ssec: StoBanditAlg}
The stochastic bandit subroutine is used to efficiently deal with inherent uncertainty in the payoff obtained upon successful matching.
In this section we develop the theory for the setting in which agents use a UCB based SB subroutine. Similar results for Thompson Sampling are supplied in the Appendix. 

To being, we denote the number of times agent \(a\) gets successfully matched with firm \(f\) till time \(t\) as \(\numMatches_{a,f}(t)\). Similarly, the number of times agent \(a\) gets collided with firm \(f\) till time \(t\) be \(\numCollide_{a,f}(t)\). Given this notation, the UCB \cite{auer2002using} estimate of agent \(a\) for every \(f\) at time \(t\) is given by
\begin{align*}
    \UCB_{a,f}(t) = \mean_{a,f}(t-1) + \sqrt{\frac{2\log(1+\totalMatch_{a}\log^2(\totalMatch_a))}{\numMatches_{a,f}(t)}},
\end{align*}
where \(\totalMatch_a(t) = \sum_{f\in \firmSet}\numMatches_{a,f}(t)\) and \(\mean_{a,f}(t-1)\) is the empirical average of the payoffs received from successfully matching to firm \(f\) until time \(t\). The UCB estimate is composed of two parts: (i) the empirical mean which captures the exploitation aspect; and (ii) exploration bonus that decreases as \(M_{a,f}(t)\) increases. We remark that it does not depend on the number of collisions $\numCollide_{a,f}(t)$.

\subsection{Adversarial Bandit Subroutine}\label{ssec: AdvBanditAlg}
A key component of the proposed methodology is to use an adversarial bandit subroutine to deal with the competitive aspect of the problem. In particular, the AB subroutine updates the request probability \((\pullProb_{a,f})_{f\in\firmSet}\) such that agent stops requesting firm on which the collisions are high (but ensures that it does not miss out on the firm if it is achievable). Intuitively, by construction, the adversarial bandit algorithm learns to prune arms on which collisions would happen frequently, and request firms where it is possible to successfully match very often. We show this by analyzing its regret and showing that high regret is incurred if  the algorithm either prunes too often when successfully matching is possible or requesting a firm that is unachievable due to the frequent presence of higher ranked agents. By bounding the regret of the AB subroutine we immediately get a bound on the number of collisions.

We now describe the update scheme for \(\pullProb_{a,f}(t)\) for any \((a,f)\) at any time \(t\in[T]\).  
In this work we consider an optimistic mirror descent based AB subroutine specialized from~\cite{bubeck2019improved}. Interestingly such AB algorithms have data dependent regret bounds \cite{wei2018more}, \cite{bubeck2019improved} unlike other AB algorithms like Exp3 \cite{lattimore2020bandit,slivkins2019introduction}. Since the competition in the matching market is not actually adversarial such data-dependent regret bounds enables us characterize the competition more effectively in the analysis than just treating competition as adversarial\footnote{We review the required background on optimistic mirror descent based AB algorithms in the Appendix along with a result which captures the characterizes the corresponding data-dependent regret bounds in the setting of matching markets. }. We note that the proof techniques developed here can also be used to analyze an Exp3 based AB subroutine but the regret bounds of such an approach will not be as sharp. 

 For a given agent $a$, our algorithm associates a separate AB subroutine to every firm $f \in \firmSet$. Each AB algorithm has \emph{two arms}  which correspond to the action of requesting the firm \(f\) or pruning it, each of which incurs different losses depending. In particular, if  \(\pullInstant_{a,f}(t)=0\) then it receives a fixed loss of 0; if \(\pullInstant_{a,f}(t)=1\) the loss received is \(+1\) or \(-1\) if it collides or matches respectively. If we denote the loss received by the AB subroutine  associated with \((a,f)\) at time \(t\) by \(\lossPull_{a,f}(t)\), we note that \(\lossPull_{a,f}(t) = \pullInstant_{a,f}(t)\lr{1-2\isMatched_{a}(t)}\). Note that \(\isMatched_a(t)\) is unknown to any agent before requesting any firm as it also depends on the requests made by other agents. 

We note that the request probability \(\pullProb_{a,f}\) is not updated at every time $t$, but only when \(\actQuery_{a,f}(t)=1\) (i.e., if all firms with a higher UCB index have been pruned). As such the adversarial bandit algorithms can be seen as operating on a randomized timescale \(\numSelect_{a,f}(T) = \{t\in[T]: \actQuery_{a,f}(t) = 1 \}
\) which are the time steps on which agent \(a\) considers firm \(f\).  We note that \(\pullProb_{a,f}(t+1)=\pullProb_{a,f}(t)\) if \(t\not\in \tau_{a,f}(T)\).

For the specific AB algorithm we analyze (which is a version of optimistic mirror descent with a log-barrier regularizer first analyzed in~\cite{wei2018more}), the simple setup of the losses leads to a closed form update for the probability of requesting or pruning a firm. In particular, for every \((a,f)\in \actSet\times \firmSet\) and \(t\in \numSelect_{a,f}(T)\), the optimistic mirror descent AB subroutine creates an unbiased estimate of the loss due to pruning and requesting as \(\lossPruneEst_{a,f}(t)\) and \(\lossPullEst_{a,f}(t)\) respectively. In particular, if \(\pullInstant_{a,f}(t)=1\)
\begin{align*}
\lossPruneEst_{a,f}(t) =\frac{1+\lossPull_{a,f}(t-1)}{2}, \quad \lossPullEst_{a,f}(t) = \frac{1-2\isMatched_a(t)-\lossPull_{a,f}(t-1)}{2\pullProb_{a,f}(t)}+\frac{1+\lossPull_{a,f}(t-1)}{2}.
\end{align*}
On the other hand, if \(\pullInstant_{a,f}(t)=0\) then 
\begin{align*}
    \lossPruneEst_{a,f}(t) = \frac{-\lossPull_{a,f}(t-1)}{2(1-\pullProb_{a,f}(t))}+\frac{1+\lossPull_{a,f}(t-1)}{2}, \quad \lossPullEst_{a,f}(t) = \frac{1+\lossPull_{a,f}(t-1)}{2}
\end{align*}

The term \(\frac{1+\lossPull_{a,f}(t-1)}{2}\) is an optimistic prediction of the losses based on the last round of interaction~\cite{bubeck2019improved}. Given these estimators the probability of requesting a firm is updated as: 
\begin{align*}
\pullProb_{a,f}(t+1)=(1-\exploit_{a,f}(t))\auxProb_{a,f}(t)+\exploit_{a,f}(t)\pullInstant_{a,f}(t),
\end{align*}
where: \[x_{a,f}(t)= \left(2+\forgetMeasure(t)-\sqrt{4+\forgetMeasure(t)^2}\right) (2\forgetMeasure(t))^{-1}\]  for $\forgetMeasure(t) = \SS \lr{ \lossPullEst_{a,f}(t) - \lossPruneEst_{a,f}(t)} + \frac{1}{\auxProb_{a,f}(t-1)}-\frac{1}{1-\auxProb_{a,f}(t-1)}$, is the result of a step of mirror descent with the log-barrier regularizer, and  \(\exploit_{a,f}(t) = \frac{\exploitSS(1-\lossPull_{a,f}(t))}{2+\exploitSS(1-\lossPull_{a,f}(t))}\), for \(\lambda>0\), promotes exploration.
The algorithmic description of this process is stated in Algorithm \ref{alg:AdaptivePart}.

\begin{algorithm}
\SetAlgoLined
\LinesNumbered
\SetKwInOut{Input}{Input}
\SetKwInOut{Parameters}{Parameters}
\SetKwInOut{Output}{Output}
    \Input{ $\pullInstant_{a,f},\auxProb_{a,f},\pullProb_{a,f},\lossPull_{a,f},\isMatched_{a}$ }
    \Parameters{ $
    \SS\leq \frac{1}{50}, \exploitSS = 8 \SS$
    }
    \If{\(\pullInstant_{a,f}= 0\)}{
    Set \(\lossPruneEst_{a,f} = \frac{-\lossPull_{a,f}}{2\lr{1-\pullProb_{a,f}}}+\frac{\lossPull_{a,f}+1}{2},\quad \lossPullEst_{a,f} =  \frac{1+\lossPull_{a,f}}{2}\)
\\
    Update \(\lossPull_{a,f} \la 0\)
    }
    \If{\(\pullInstant_{a,f}=1\)}{
    Set \(\lossPruneEst_{a,f} =\frac{1+\lossPull_{a,f}}{2}, \quad \lossPullEst_{a,f} = \frac{1-2\isMatched_a-\lossPull_{a,f}}{2\pullProb_{a,f}}+\frac{1+\lossPull_{a,f}}{2}\) 
    \\
    Update \(\lossPull_{a,f} \la 1-2\isMatched_a\)
    }
    Set \(\forgetMeasure = \SS \lr{ \lossPullEst_{a,f} - \lossPruneEst_{a,f}} + \frac{1}{\auxProb_{a,f}}-\frac{1}{1-\auxProb_{a,f}}\) 
    \\
    Update \(\auxProb_{a,f} \la \frac{2+\forgetMeasure-\sqrt{4+\forgetMeasure^2}}{2\forgetMeasure}\) and 
    set \(\exploit_{a,f} = \frac{\exploitSS(1-\lossPull_{a,f})}{2+\exploitSS(1-\lossPull_{a,f})}\) 
    Update \(\pullProb_{a,f} \la (1-\exploit_{a,f})\auxProb_{a,f}+\exploit_{a,f}\pullInstant_{a,f}\)\\
  \Output{ \(\lossPull_{a,f},\auxProb_{a,f},\pullProb_{a,f}\)}
  \caption{\pullModule}
    \label{alg:AdaptivePart}
\end{algorithm}

\section{Bounds on the regret of proposed algorithm}\label{ssec: RegretBound}

To capture the performance of the algorithm we use the natural notion of \emph{stable regret} as introduced in \cite{liu2020competing}. 
More formally, the stable regret accrued by any agent \(a\in\actSet\) is
\begin{align}\label{eq: RegretDefinition}
\avg[\regret_a(\horizon)] &= \avg\ls{\sum_{t=1}^{T}\utilityAgent_{a,\stableArm_a}-\sum_{t=1}^{T}\utilityAgent_{a,\chosenFirm_a(t)}}\leq\sum_{f\in \subArm_a} \gap_{a}(f) \avg[\numMatches_{a,f}(T)] + \utilityAgent_a(\stableArm_a)\sum_{f\in\mc{F}}\avg[\numCollide_{a,f}(\horizon)],
\end{align}
where \(\gap_{a}(f) = \utilityAgent_{a}(\stableArm_a)-\utilityAgent_{a}(f) \) is the gap between the mean that agent $a$ gets upon successfully matching with its stable match as compared firm \(f\).  If there are no collisions, then this regret definition is same as that used in stochastic bandits literature (\cite{lattimore2020bandit}). In the following theorem, we present the regret of any agent using Algorithm \ref{alg:prunedUCBFinal}:

\begin{theorem}\label{thm: UCBMainPaper} Suppose every agent \(a\in\actSet\) uses Algorithm \ref{alg:prunedUCBFinal}.
Then for any \(i\in[ \numMarket]\) :
\begin{align*}
    \sum_{j=1}^{i}\sum_{a\in\subActSet_j}\avg[\regret_a(T)] = \bigo\lr{C_i|\firmSet|{|\actSet|}{{\log(T)}\lr{1+\frac{1}{\gap^2}} } }
\end{align*}
where \(\gap=\min_{a,f}\gap_{a,f}\) and \(C_i\) is a constant dependent on market \(\market_i\) and \(C_1<C_2<...<C_{\numMarket}\).
\end{theorem}
We see that the regret of any agent \(a\in \actSet\) is logarithmic in horizon \(T\), which matches the lower bound for single player stochastic bandit algorithms \cite{lai1985asymptotically}. As such, perhaps surprisingly, we observe that in $\alpha$-reducible markets, it is possible for agents to learn while competing without incurring drastically worse regret in the long run. It is interesting to note that the learning of agent depends on its position in the market as per preferences (Remark \ref{rem: MarketDecomp}). An agent low in the hierarchy incurs more regret during the learning process due to the agents higher up in the hierarchy driven mainly by the larger number of collisions incurred while waiting for agents higher in the hierarchy to stop exploring. We note that in the worst case the constant \(C_i\) can grow exponentially in the number of agents in the market. We note that this is a consequence of the proof technique and not fundamental limitation of the algorithmic design paradigm as we show  through numerical studies in next section. We leave this as a future work to establish tighter regret bounds in terms of number of agents.
In the Appendix we also show that in Algorithm \ref{alg:prunedUCBFinal} if we use a SB subroutine  based on Thompson Sampling then a similar regret guarantee can be obtained. We now present a sketch of the proof of Theorem \ref{thm: UCBMainPaper}.


\paragraph{Sketch of the proof.} Before presenting the sketch, we first define few notations that would make the exposition clear. Let \(\numMatches_{a,\subArm_a}(T)=\sum_{f\in\subArm_a}\numMatches_{a,f}(T), \numMatches_{a,\superArm_a}(T)=\sum_{f\in\superArm_a}\numMatches_{a,f}(T)\). Moreover, for any \(a\in\actSet\) define 
\(
    H_{a,\stableArm_a}(t)=\{\exists a'\in \actSet \ \text{s.t.} \ \utilityFirm_{\stableArm_a}(a')\geq \utilityFirm_{\stableArm_a}(a), \chosenFirm_{a'}(t)=f \}
\)
which is an event that characterizes if any other more preferred agent has requested the stable match of agent \(a\) at time \(t\). Against the preceding backdrop, we now present the following crucial lemma: 
\begin{lemma}\label{lem: MainLemma}
Suppose every agent uses Algorithm \ref{alg:prunedUCBFinal} then the following holds:
\begin{itemize}[leftmargin=20pt]
\item[\textbf{(L1)}] For any \(i\in[\numMarket]\), the cumulative regret can be decomposed as 
\begin{align*}
    \sum_{j=1}^{i}\sum_{a\in\subActSet_j}\avg[\regret_a(T)]&=  \bigo\bigg(\sum_{i=1}^{k}\sum_{a\in\subActSet_i}(\avg[\numMatches_{a,\subArm_a}(T)] +  \sum_{\underset{f\neq  \{f^*_a\}}{f\in F}} \avg[C_{a,f}(T)]+ \avg[\sum_{t=1}^{T}H_{a,f^*_a}(t)])\bigg);
\end{align*}
    \item[\textbf{(L2)}] For any \(i\in [\numMarket]\), the expected matches with suboptimal firm satisfies
    \begin{align*}
       \sum_{j=1}^{i}\sum_{a\in\subActSet_j}\avg[\numMatches_{a,\subArm_a}(T)] = \bigo\lr{  \sum_{j=1}^{i}\sum_{a\in\subActSet_j}\lr{|\subArm_a|\log(T)\lr{1+\frac{1}{\gap^2}} +\avg\ls{\sum_{t=1}^{T}\collisionEvent_{a,\stableArm_a}(t)}}}
    \end{align*}
    \item[\textbf{(L3)}] The expected number of collisions between for any agent \(a\in\actSet\) satisfies
    {
\begin{align*}
    \sum_{f\in{\firmSet}}\avg[\numCollide_{a,f}(T)] = \bigo\lr{|{\firmSet}|\log(T)+\avg\ls{\numMatches_{a,\subArm_a}(T)+\numMatches_{a,\superArm_a}(T)+ \sum_{t=1}^{T}\one\lr{\collisionEvent_{a,\stableArm_a}(t) }}};
\end{align*}
}

\item[\textbf{(L4)}] For any \(i\in [\numMarket]\) we have 
\begin{align*}
    \sum_{j=1}^{i}\sum_{a\in\subActSet_j}\avg\ls{ \sum_{t=1}^{T}\one\lr{\collisionEvent_{a,\stableArm_a}(t)} }=  \bigo\lr{ C_i\lr{\sum_{j=1}^{i}|\subActSet_j|}\log(T)\lr{1+\frac{1}{\gap^2} }},
\end{align*}
where \(C_i\) is a constant dependent on market \(\market_i\) such that \(C_1<C_2<...<C_{\numMarket}\). 
\item[\textbf{(L5)}] For any \(i\in[\numMarket]\) we have 
\begin{align*}
    \sum_{j=1}^{i}\sum_{a\in\subActSet_j}\sum_{f\in\superArm_a}\avg[\numMatches_{a,f}(T)] \leq \bigo\lr{C_i \lr{\sum_{j=1}^{i}|\subActSet_j|}|\firmSet|\log(T)\lr{1+\frac{1}{\gap^2}} }
\end{align*}
\end{itemize}
\end{lemma}

Theorem \ref{thm: UCBMainPaper} is proved using \textbf{(L1)}-\textbf{(L5)} from Lemma \ref{lem: MainLemma}. Note that \textbf{(L1)} follows from \eqref{eq: RegretDefinition} and the definition of \(\collisionEvent_{a,\stableArm_a}(t)\). From \textbf{(L1)} we see that to bound the regret we need to consider three components: (i) expected number of matchings with suboptimal firms, (ii) expected number of collisions with any firm other than stable match, (iii) the \emph{potential collisions} at the stable match\footnote{by potential collision at stable match we mean total number of collision that would have been faced by an agent at its stable firm had it always requested the stable firm}. \textbf{(L2)} bounds the expected number of matchings with suboptimal firms. Note that the total matchings between agent \(a\) and firm \(f\) is \(\numMatches_{a,f}(T) = \sum_{t=1}^{T}\one\lr{\isMatched_{a}(t)=1,\chosenFirm_a(t)=f}\). Thus, we present the following lemma which plays a key role in the proof of \textbf{(L2)}:
\begin{lemma}
\label{lem: ChooseF}
   The event that agent \(a\) chooses the firm \(f\in \subArm_a\) and successfully matches at  time \(t\in[T]\) satisfies 
    \begin{align*}
    \{\isMatched_a(t)=1, \chosenFirm_a(t)=f\} \subset \lb{\isMatched_a(t)= 1,\UCB_{a,\stableArm_a}(t)\leq \UCB_{a,f}(t) }\cup \{\actChoose_{a,f}(t) = 1,\actChoose_{a,\stableArm_a}(t)=0\}
    \end{align*}
\end{lemma}
Lemma \ref{lem: ChooseF} separates the challenge associated with uncertainty and that of competition.
Note that the first event on the right hand side is the one which is standard to the analysis of UCB algorithm (\cite{lattimore2020bandit}). Meanwhile, the other event corresponds to the case when the stable firm is pruned by agent \(a\) in order to avoid potential collisions. To bound latter event we use the regret bounds for the adversarial bandit subroutine  (refer to Appendix). 

To bound \textbf{(L3)} we use the path length based regret bounds from \cite{bubeck2019improved},  \cite{wei2018more} for the adversarial bandit subroutine. Meanwhile to bound \textbf{(L4)} we use the \(\alpha-\)reduciblity assumption and \textbf{(L2)}. In particular, the \(\alpha-\)reduciblity assumption induces a hierarchy in the market as per Remark \ref{rem: MarketDecomp}. This decomposition reduces the bound in \textbf{(L4)} to appropriate accounting of number of matches with suboptimal firms via an induction argument. Finally, \textbf{(L5)} follows again due to hierarchy induced by \(\alpha-\)reducibility and using \textbf{(L2)-(L4)}.

\section{Experimental Study}\label{sec: Numerics}
In this section we present the numerical experiments that demonstrates and validates the results presented in this paper. Moreover, we also observe that our algorithm performs surprisingly well in general market structure, that is in markets which are not \(\alpha-\)reducible. We leave this as a future work to establish the regret bounds for the proposed algorithms in general markets. 

In both sets of experiments, we consider a market comprising of 5 agents and 5 firms. We consider the following two settings:

\textbf{(S-I).}  randomly initialized preference for agents and randomly initialized (but uniform) preference for firms. This setting ensures that market is \(\alpha-\)reducible 

\textbf{(S-II).} randomly initialized preference for agents and firms. In this part we specifically consider setting where \(\alpha-\)reducibility does \emph{not} hold. This would provide directions for future research in this area. 

In our simulations for every agent we randomly sample the preference ordering of firms and assign a mean reward in \([0,5]\) such that the successful match with most preferred firm gives mean reward \(5\) and the least preferred firm gives the mean reward \(0\) and the mean rewards from other firms are equally spaced between \([0,5]\). The rewards follow a normal distribution with variance 1.  We run both Algorithm \ref{alg:prunedUCBFinal} and Algorithm \ref{alg:prunedTSFinal} for 25 times for two randomly sampled preference ordering for each of \textbf{(S-I)-(S-II)}.

In Figure \ref{fig: Alpha} we consider \textbf{(S-I)} and observe the performance of algorithms. We observe that the mean regret (taken over 25 runs) accumulated by the algorithms saturate very quickly and agents identify their stable match. 
In Figure \ref{fig: noAlpha} we consider \textbf{(S-II)} and observe the performance of algorithm. Surprisingly, even without the \(\alpha-\)reducibility structure, the mean regret\footnote{mean regret here refers to the agent-optimal stable regret\cite{liu2021bandit}} (taken over 25 runs) accumulated by the algorithms saturate very quickly and agents identify their stable match.  This presents an opportunity to further explore the algorithm presented in this paper for general markets. 

 Furthermore, in both \textbf{(S-I)-(S-II)} we observe that the TS-DMA has higher variance but is faster than UCB-DMA. This is because, compared to UCB-DMA, we observe empirically that TS-DMA very rarely encounters the scenario where all of the firms gets pruned by the adversarial bandit module. We would also like to point that in some cases the regret can be negative (which is desirable) as is shown in Figure \ref{fig: Alpha}(c) for the red agent.

{
\begin{figure}[h]
    \begin{subfigure}{.48\textwidth}
    \centering
    \includegraphics[width=.8\linewidth]{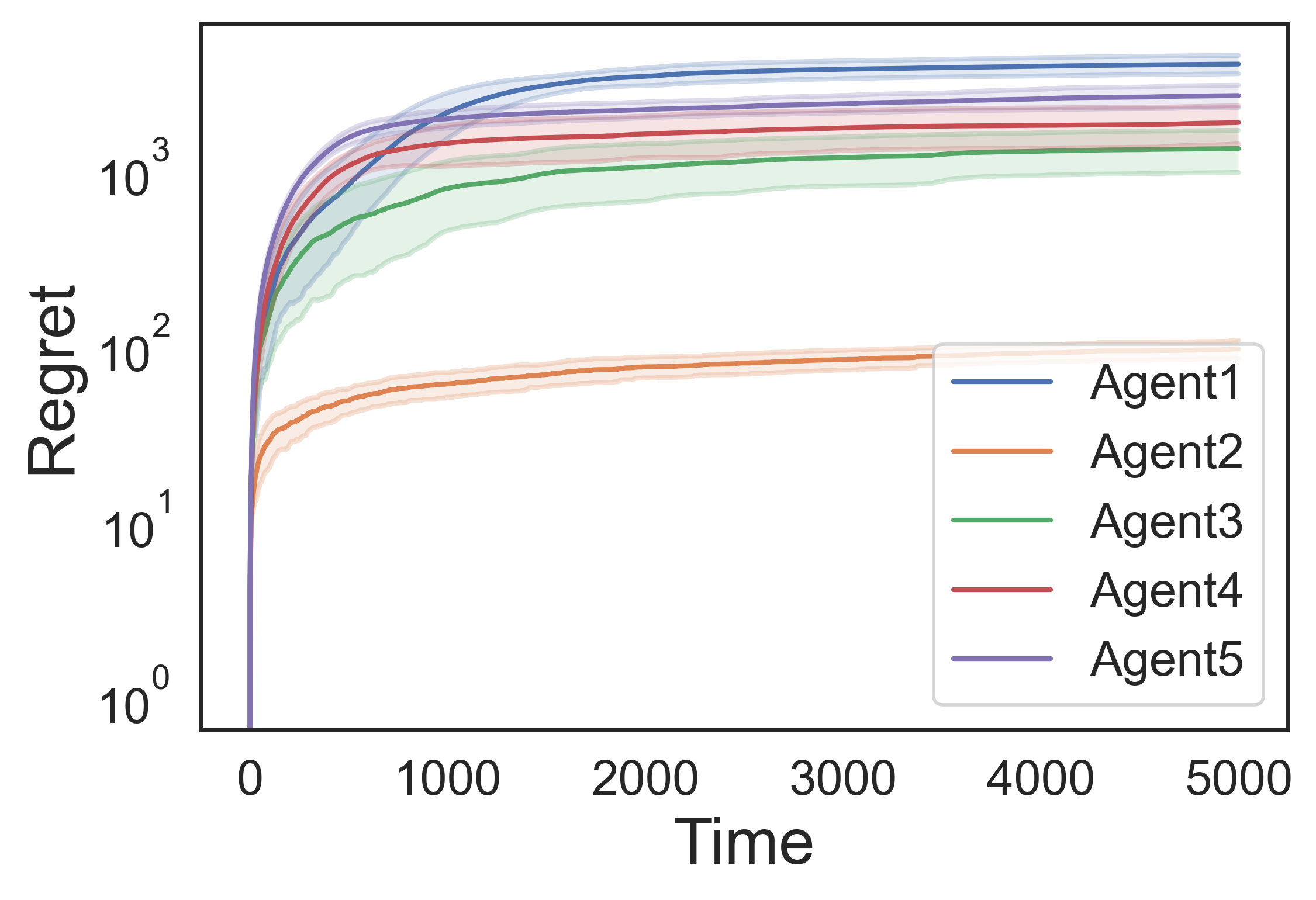}  
    \caption{UCB-DMA(Algorithm \ref{alg:prunedUCBFinal})}
    \label{fig:sub-first}
    \end{subfigure}
    \begin{subfigure}{.48\textwidth}
    \centering
    \includegraphics[width=.8\linewidth]{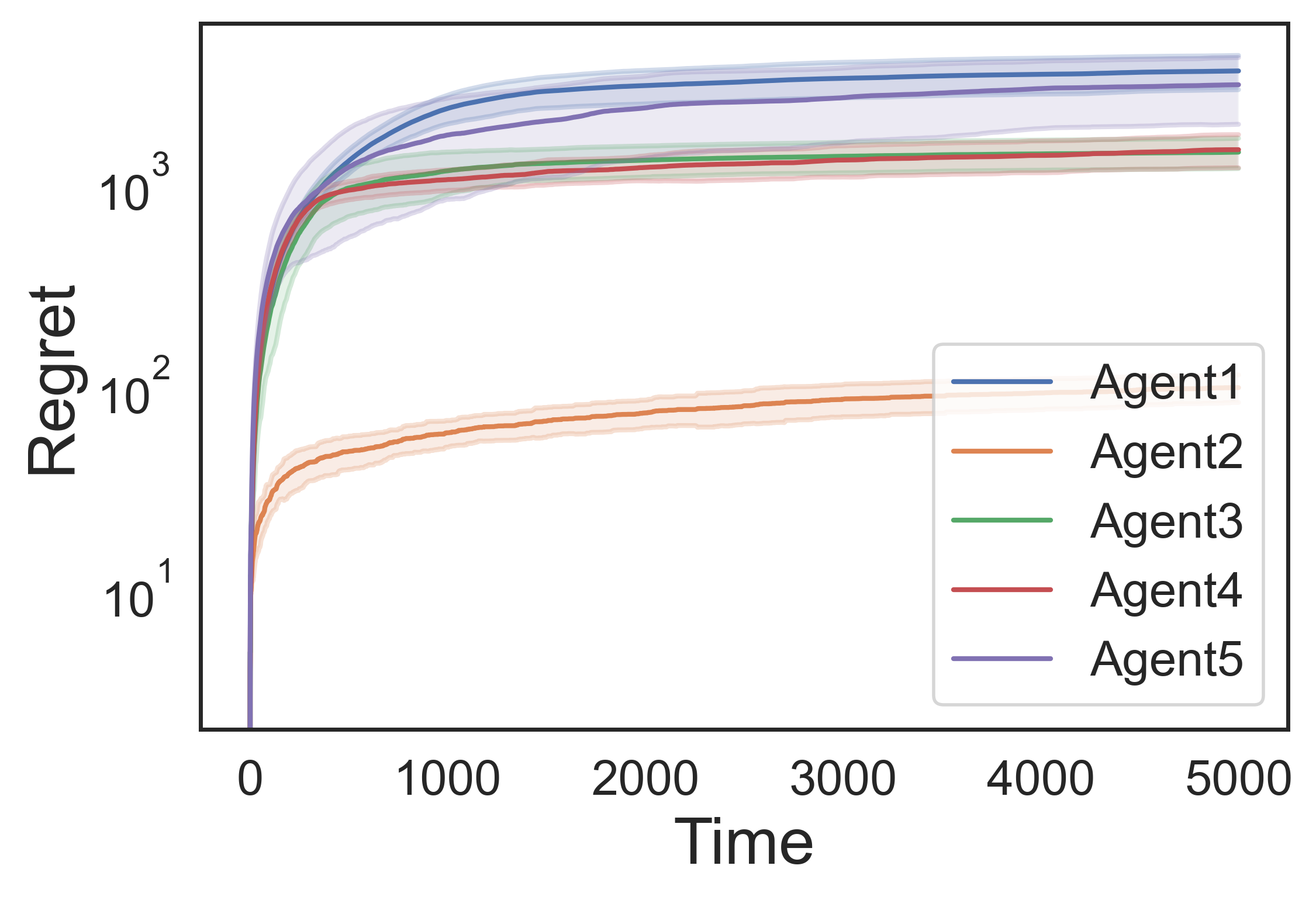}  
    \caption{UCB-DMA(Algorithm \ref{alg:prunedUCBFinal})}
    \label{fig:sub-first}
    \end{subfigure}
    \newline 
    \begin{subfigure}{.48\textwidth}
    \centering
    \includegraphics[width=.8\linewidth]{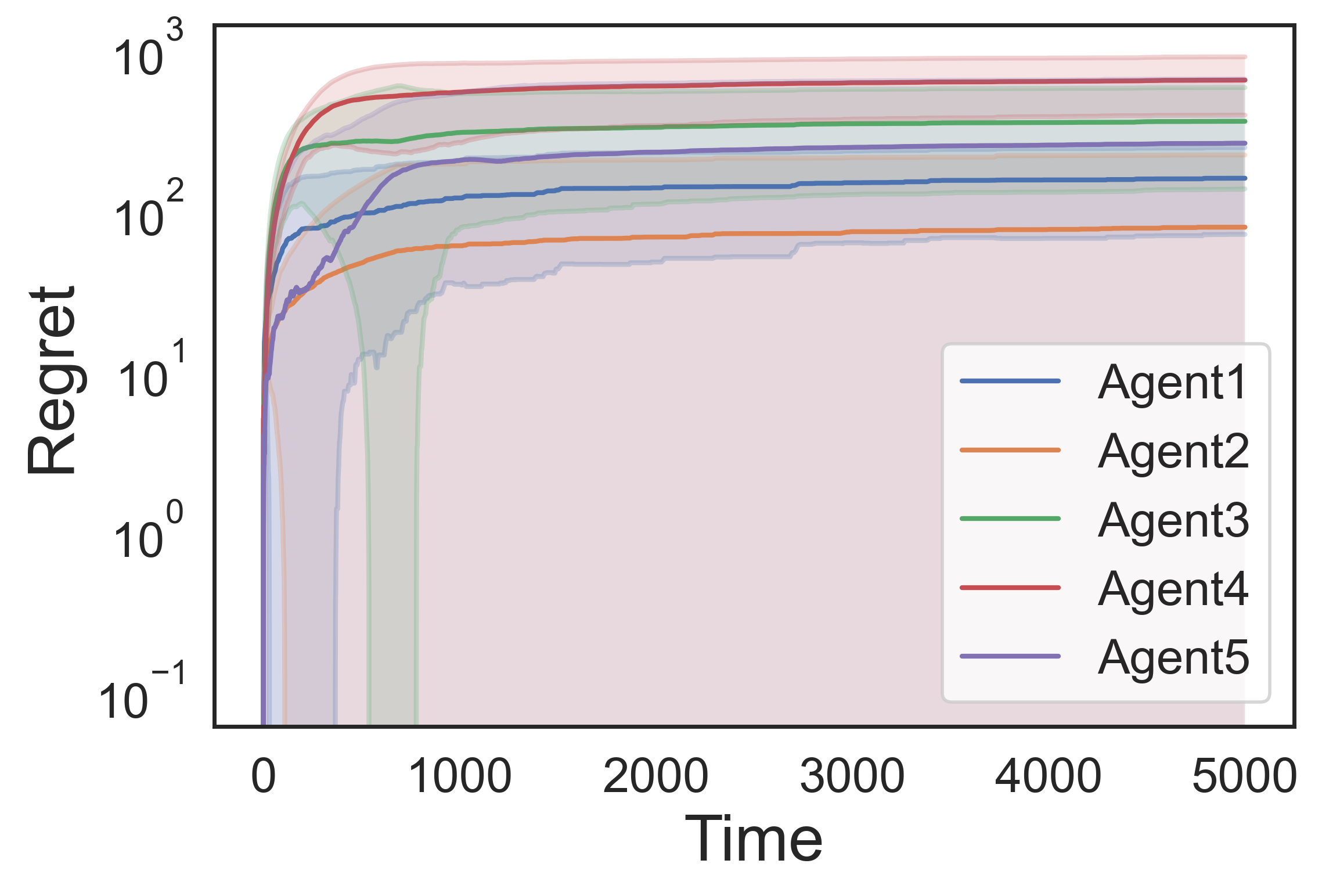}  
    \caption{TS-DMA(Algorithm \ref{alg:prunedTSFinal})}
    \label{subfig:TS_Alpha}
    \end{subfigure}
    \begin{subfigure}{.48\textwidth}
    \centering
    \includegraphics[width=.8\linewidth]{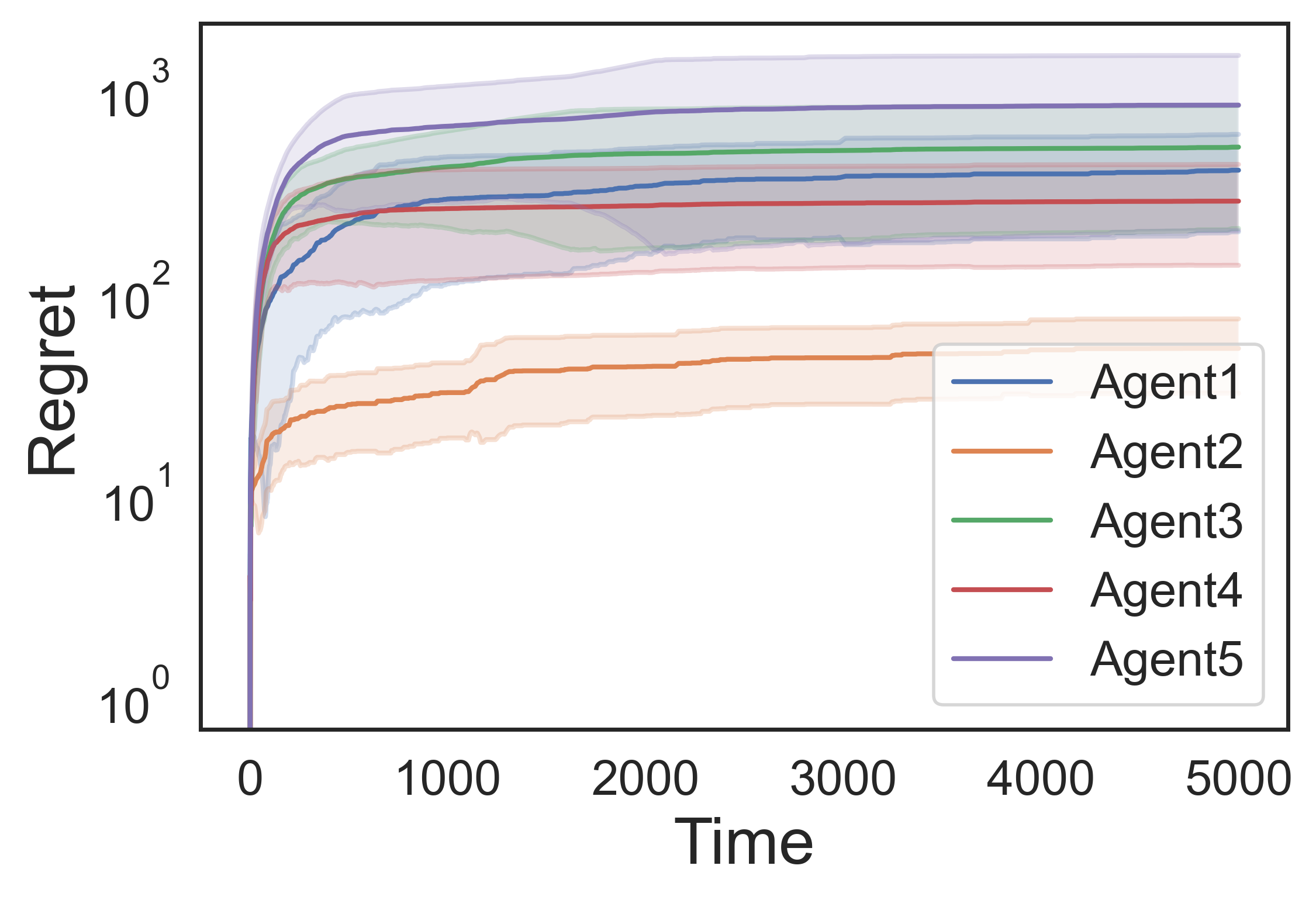}  
    \caption{TS-DMA(Algorithm \ref{alg:prunedTSFinal})}
    \label{fig:sub-first}
    \end{subfigure}
    \caption{Performance of UCB-DMA (Algorithm \ref{alg:prunedUCBFinal}) and TS-DMA(Algorithm \ref{alg:prunedTSFinal}) where \(\alpha-\)reducibilty condition is satisfied. We simulated the algorithm for two randomly generated preference orderings which satisfy the \(\alpha\)-reducibility condition. The simulation results of one of the preference ordering are presented in left column and for the other in right column. The bold lines and the corresponding shaded region denotes the mean regret and the variance of regret for the agents over 25 runs of the algorithm.   }
    \label{fig: Alpha}
\end{figure}}

\begin{figure}[h]
    \begin{subfigure}{.48\textwidth}
    \centering
    \includegraphics[width=.8\linewidth]{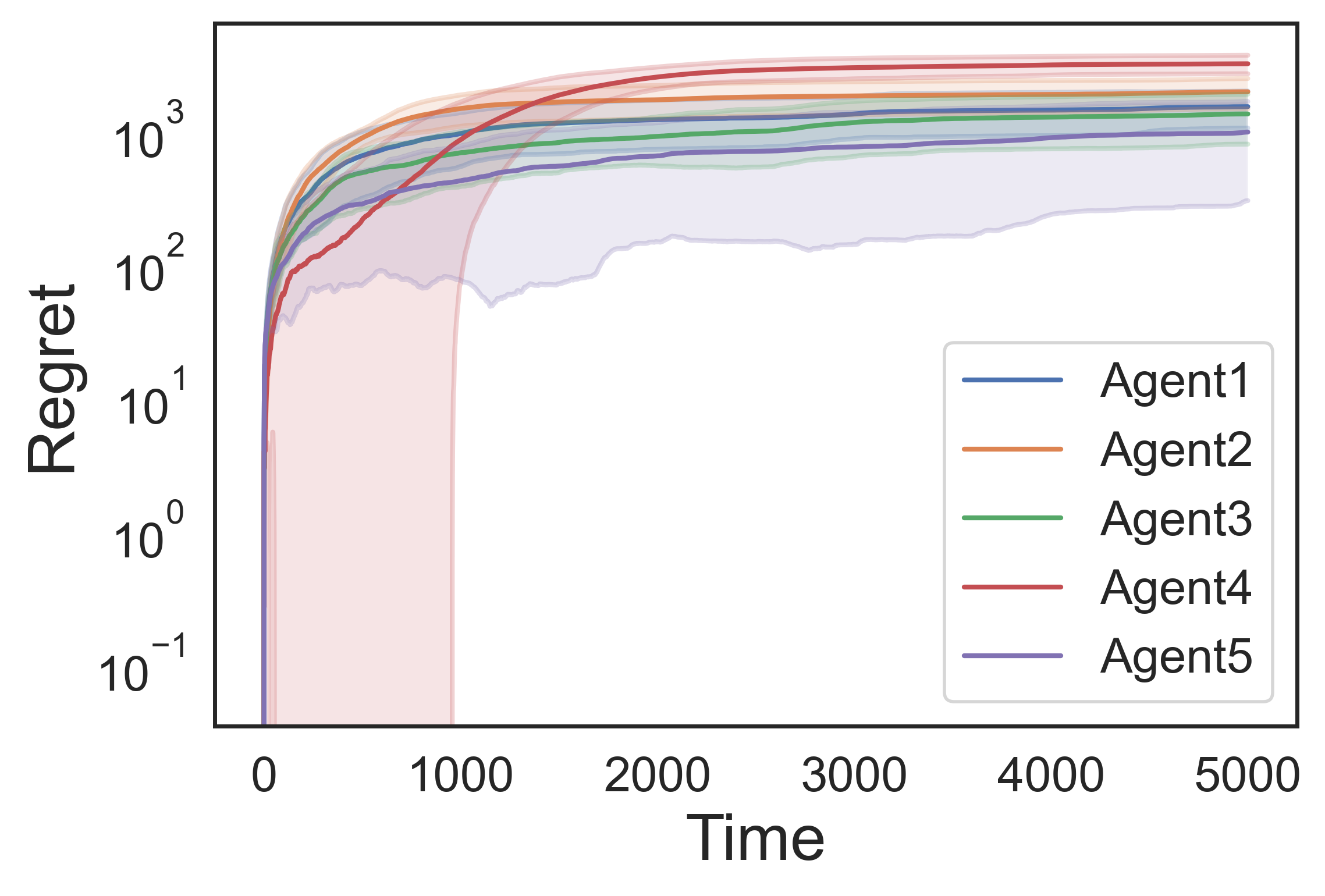}  
    \caption{UCB-DMA(Algorithm \ref{alg:prunedUCBFinal})}
    \label{fig:sub-first}
    \end{subfigure}
    \begin{subfigure}{.48\textwidth}
    \centering
    \includegraphics[width=.8\linewidth]{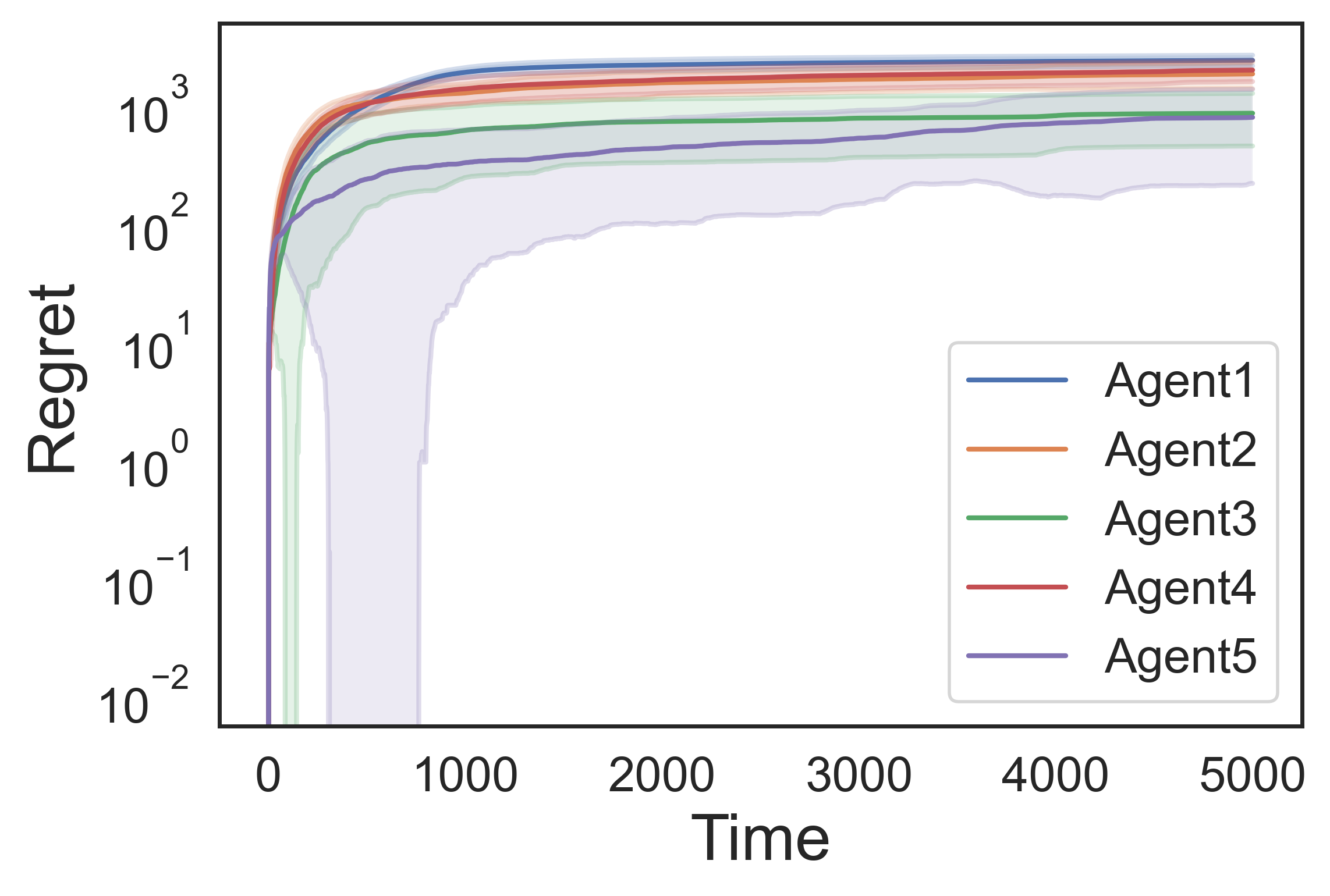}  
    \caption{UCB-DMA(Algorithm \ref{alg:prunedTSFinal})}
    \label{fig:sub-first}
    \end{subfigure}
    \newline 
    \begin{subfigure}{.48\textwidth}
    \centering
    \includegraphics[width=.8\linewidth]{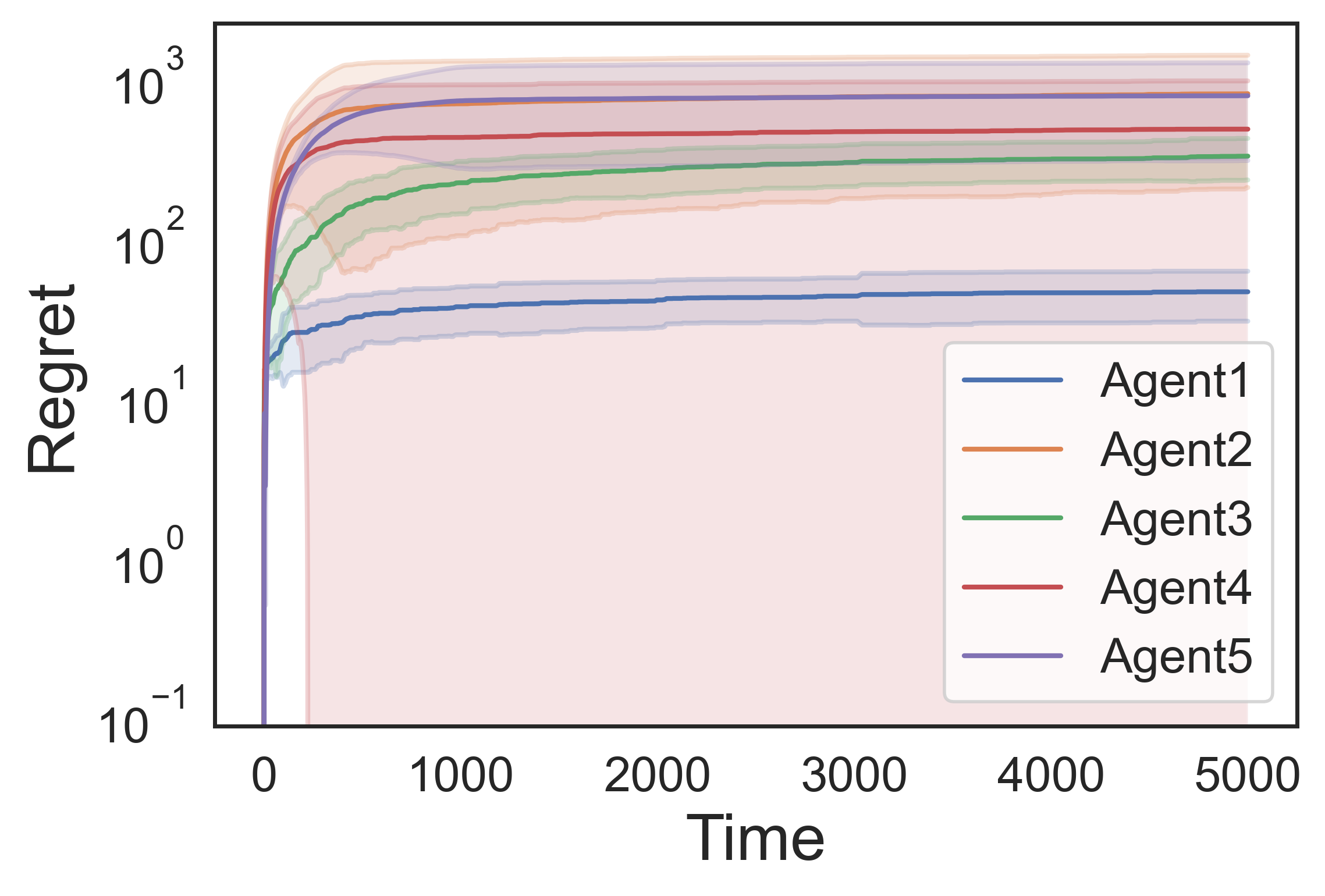}  
    \caption{TS-DMA(Algorithm \ref{alg:prunedTSFinal})}
    \label{fig:sub-first}
    \end{subfigure}
    \begin{subfigure}{.48\textwidth}
    \centering
    \includegraphics[width=.8\linewidth]{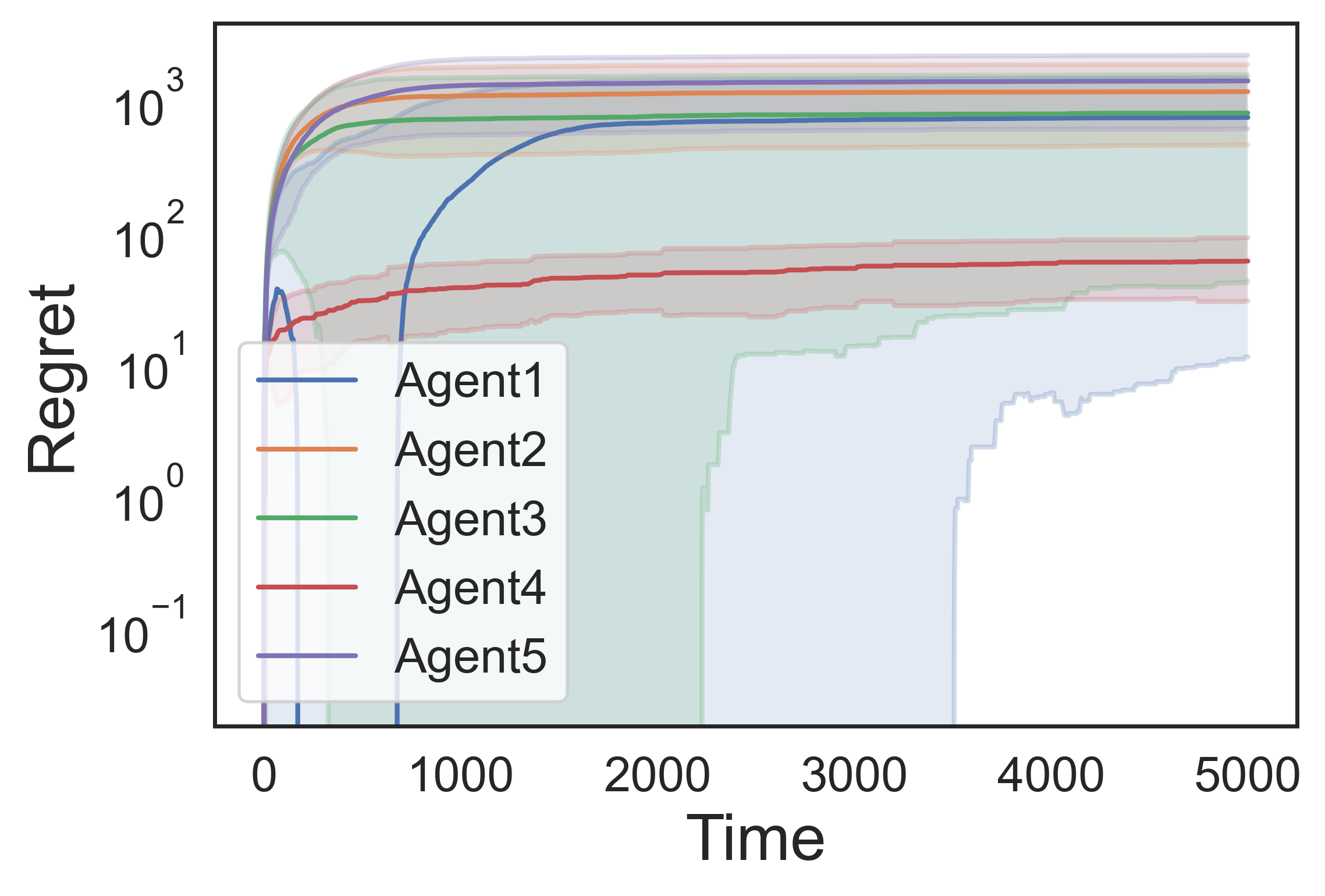}  
    \caption{TS-DMA(Algorithm \ref{alg:prunedTSFinal})}
    \label{fig:sub-first}
    \end{subfigure}
     \caption{Performance of UCB-DMA (Algorithm \ref{alg:prunedUCBFinal}) and TS-DMA(Algorithm \ref{alg:prunedTSFinal}) where \(\alpha-\)reducibilty condition is NOT satisfied. We simulated the algorithm for two randomly generated preference orderings which satisfy the \(\alpha\)-reducibility condition. The simulation results of one of the preference ordering are presented in left column and for the other in right column. The bold lines and the corresponding shaded region denotes the mean regret and the variance of regret for the agents over 25 runs of the algorithm.   }
    \label{fig: noAlpha}
\end{figure}

\section{Conclusions}\label{sec: Conclusion}
We consider a problem of bandit learning in two-sided matching markets comprising of agents and firms. We consider the setting where agents have unknown preferences over the firms. In this paper we present simple design  principle for decentralized, communication and coordination free algorithm for learning in two-sided matching markets. The primary challenge in learning in two-sided matching market is to balance exploration, exploitation and collision avoidance. We embed the aforementioned properties in the algorithm by a novel idea of blending a stochastic bandit subroutine  with an adversarial bandit subroutine . The stochastic bandit subroutine  is required for balancing the exploration-exploitation trade-off while the adversarial bandit subroutine  limits the collisions. As an instance of this design principle, we present an algorithm which has the stochastic bandit subroutine based on UCB and the adversarial bandit subroutine based on  Optimistic Mirror Descent algorithm. We show that if the preferences of agents satisfy certain structure known as $\alpha$-reducibility, then these algorithms incur a regret which is logarithmic in the time horizon. Two immediate directions of future work include: (i) extension of theoretical guarantees to general markets, and (ii) improving the dependence of regret bound on the number of agents.

\section*{Acknowledgements}
Research was partially supported by NSF under grant DMS 2013985 THEORINet: Transferable, Hierarchical,
Expressive, Optimal, Robust and Interpretable Networks and U.S. Office of Naval Research MURI
grant N00014-16-1- 2710.
\bibliography{refs}
\bibliographystyle{alpha}

\appendix

~\\
\centerline{{\fontsize{18}{13.5}\selectfont \textbf{Appendix}}}

\vspace{6pt}

In Section A, we review the adaptive adversarial algorithms proposed in \cite{bubeck2019improved} and specialize the regret bounds in the setup of this paper. In Section B we provide the proof of lemmas stated in Section \ref{ssec: RegretBound}. In Section C we provide proof of the main theorem of this paper stated in Section \ref{ssec: RegretBound}. In Section E we provide the Thompson sampling based variant of the Algorithm \ref{alg:prunedUCBFinal} and provide the analogous result as in Section \ref{ssec: RegretBound}. In Section F we provide a  table of notations for ease of reference.

\section{Adaptive Adversarial Algorithms}\label{appsec: Adaptive adversarial algorithms}
In this work we deploy the optimistic mirror descent based adversarial bandit module. We adapt algorithms from \cite{bubeck2019improved}, who improve the algorithm originally proposed in \cite{wei2018more}. In this section we recap the results from \cite{bubeck2019improved}. For the sake of completeness we restate the problem formulation and algorithm here. Towards the end we will specialize their results in the setting of this paper and state an useful result which presents the regret of such algorithms, in the context of the bandit structure described in Sec \ref{ssec: AdvBanditAlg}, in terms of the number of matchings and collisions.  

\subsection{Problem formulation from \cite{bubeck2019improved}}\label{ssec: ProblemBubeck} In this section we review algorithm described in \cite{bubeck2019improved} which is an improvement over the one described in \cite{wei2018more}. Consider a multi-armed bandit problem that proceeds in \(\tau\) time steps with \(A \leq \tau\) fixed actions. In each round \(t\), the algorithm selects one arm \(i(t)\in [A]\) and simultaneously an adversary decides the loss vector \(\ell(t)=(\ell_i(t))_{i\in[A]}\in [-1,1]^A\). Note that the adversary can be an adaptive one in that it can base its actions on the past rounds of algorithm's actions. The goal of the algorithm is to minimize the gap between total accumulated loss and the loss of best fixed arm in hindsight:
\begin{align*}
    \regretAdv(\tau) = \max_{i^\star \in [A]} \avg\ls{ \sum_{t=1}^{\tau}\ell_{i(t)}(t) - \sum_{t=1}^{\tau} \ell_{i^\star}(t) }.
\end{align*}

The algorithm is based on the optimistic mirror descend framework.  At any time \(t\), the algorithm samples an arm \(i(t)\in [A]\) with probability \(p(t)\in \Delta([A])\). The algorithm only receives the loss for the action taken and not other actions. Therfore, upon receiving the loss \(\ell_{i(t)}(t)\) the algorithm creates an unbiased estimator of losses for other actions. The estimator is
\begin{align*}
    \hat{L}_i(t) = \frac{\ell_i(t)-L(t-1)}{2p_i(t)}\one\lr{i(t)=i} + \frac{1+L(t-1)}{2}, \quad \forall \ i
\end{align*}
The unbiased loss estimate \(\hat{L}(t)\) is used to update the an auxiliary probability distribution \(x(t+1)\in \Delta([A])\) through an optimistic mirror descend update with learning rate \(\eta\). The optimistic mirror descend update is constructed from the Bregman divergence\footnote{Bregman divergence between two point \(x,y\) with respect to a convex regularizer \(\psi\) is given as \(D_{\psi}(x,y) = \psi(x)-\psi(y)-\lara{\nabla \psi(y),x-y}\).} associated with a log-barrier regularizer \(\R^{A}\ni x\mapsto \psi(x)=\frac{1}{\eta}\sum_{i=1}^{A}\ln\frac{1}{x_i}\) as follows 
\begin{align*}
    x(t+1) = \arg\min_{z\in\Delta([A])} \lara{z,\hat{L}(t)}+D_{\psi}(z,x(t)).
\end{align*}

The distribution \(x(t+1)\) is used to update the arm sampling distribution \(p(t+1)\) after mixing a small bias towards most recently picked arm as follows 
\begin{align*}
    p(t+1) = (1-\lambda(t+1))x(t+1)+\lambda(t+1)\textbf{e}_{i(t)}
\end{align*}
where \(\textbf{e}_{i^t}\in\R^{A}\) is an element of standard basis in \(\R^A\) with \(i(t)\) element as 1 and all others as zero and \(\lambda(t+1) = \frac{\lambda(1-L(t))}{2+\lambda(1-L(t))}\) for some \(\lambda>0\).

\begin{algorithm}
\SetAlgoLined
\LinesNumbered
\SetKwInOut{Input}{Input}
\SetKwInOut{Parameters}{Parameters}
\SetKwInOut{Output}{Output}
    \Parameters{\(\eta,\lambda\in (0,1), p(1),x(1)=\textsf{Unif}([A])\), \(\psi(x)=\frac{1}{\eta}\sum_{i=1}^{A}\ln\frac{1}{x_i}\)}
    \For{\(t=1,2,..,\tau\)}{
    Play \(i(t)\sim p(t)\) and observe \(L(t) = \ell_{i(t)}(t)\) \\
    Construct an unbiased estimator \(\hat{L}_i(t) = \frac{\ell_i(t)-L(t-1)}{2p_i(t)}\one\lr{i(t)=i} + \frac{1+L(t-1)}{2}\) for all \(i\in[A]\)
    \\
    Update \(x(t+1) = \arg\min_{z\in\Delta([A])} \lara{z,\hat{L}(t)}+D_{\psi}(z,x(t))\)\\
    \(p(t+1) = (1-\lambda(t+1))x(t+1)+\lambda(t+1)\textbf{e}_{i(t)}\) where \(\lambda(t+1) = \frac{\lambda(1-L(t))}{2+\lambda(1-L(t))}\)
    }
  \caption{Optimistic Mirror Descend based Adversarial Bandit Algorithm}
    \label{alg:BubeckAlg}
\end{algorithm}

Against the preceding backdrop, we restate Theorem 2 from \cite{bubeck2019improved} below: 
\begin{theorem}\label{thm: BubeckRecast}
Algorithm \ref{alg:BubeckAlg} with \(\eta\leq \frac{1}{50}\), \(\lambda =8\eta\) ensures that 
\begin{align*}
    \regretAdv(\tau) = \bigo\lr{\frac{A\ln(T)}{\eta}} +8\eta \avg\ls{V(T)}
\end{align*}
where \(V(T)\defas \sum_{t=2}^{T}|\ell_{i(t-1)}(t)-\ell_{i(t-1)}(t-1)|\) is commonly referred as ``path-length''.
\end{theorem}
\begin{remark}
Note that Theorem 2 in \cite{bubeck2019improved} requires\footnote{Moreover, it is an algebraic exercise to establish that \(\eta<\frac{1}{24}\) and \(\lambda = \frac{1-12\eta -c\cdot\sqrt{1-24\eta}}{24}\) also works for some \(c\in (0,1)\). But we don't go in this direction to retain simplicity of algorithmic description.}.  But in fact the proof goes through for \(\eta \leq 1/50\).   \(\eta\leq 1/162\) and \(\lambda = 8\eta\). This is because in \cite{bubeck2019improved} for the proof of Theorem 2, they directly lift  \cite[Theorem 7]{wei2018more}  where \(\eta\leq 1/162\) which is not  tuned efficiently. 
\end{remark}
\subsection{Adaptive Adversarial Module}\label{ssec: ApplicationBubeck}
In this section we describe \pullModule\ in Algorithm \ref{alg:prunedUCBFinal} which is based on the algorithm presented in Sec ~\ref{ssec: ProblemBubeck}.

For any \((a,f)\in \actSet\times \firmSet\), the adversarial bandit module associated with \((a,f)\) ( as described in Algorithm \ref{alg:AdaptivePart} ) is a version of Algorithm \ref{alg:BubeckAlg} for case when there are two actions: \emph{request the firm \(f\)} or \emph{prune the firm \(f\)}. In addition, the loss incurred due to pruning the firm \(f\) is always 0 while the loss incurred due to pulling an firm \(f\) depends on whether the agent \(a\) got matched with it or collided with it.
In this special case of two actions, the optimistic mirror descent update (line 4 in Algorithm \ref{alg:BubeckAlg}) can be obtained in closed form (see Lemma \ref{lem: TwoArmBregman}). 
Note that the adversarial bandit module associated with any agent-firm tuple \((a,f)\) is only used when \(t\in \tau_{a,f}(T)\subset[T]\).  
\begin{lemma}\label{lem: AppRegredAdv}
Given a scalar \(\eta\leq \frac{1}{50}\), for any agent-firm pair \((a,f)\in\actSet\times \firmSet\), the regret of the adversarial bandit algorithm is bounded as 
\[
\avg[\regretAdv_{a,f}(\numSelect_{a,f}(T))] \leq   \bigo\lr{\frac{\log(T)}{\eta}} + 32\eta \avg\ls{\min\lb{\numPotMatches_{a,f}(T),\numPotCollide_{a,f}(T),\numMatches_{a,f}(T)+\numCollide_{a,f}(T)}},
\]
where \(\numPotMatches_{a,f}(T) = \sum_{t=1}^{T}\one\lr{\collisionEvent^{\comp}_{a,f}(t)}\) and \(\numPotCollide_{a,f}(T) = \sum_{t=1}^{T}\one\lr{\collisionEvent_{a,f}(t)}\).
\end{lemma}
\begin{proof}
To prove this lemma we only need to bound the path length \(V_{a,f}(T)\) in Theorem \ref{thm: BubeckRecast}. We claim that the path length \(V_{a,f}(T) \leq \min\{\numPotCollide_{a,f}(T),\numPotMatches_{a,f}(T)\}\). Recall \(\tau_{a,f}(T)=\{t\in T: \actQuery_{a,f}(t)=1\}\). For the remaining proof for any \(t\in\tau_{a,f}(T)\) by \(t-1\) we mean \(\max\{\mathfrak{t}<t:\mathfrak{t}\in\tau_{a,f}(T)\}\). 
For any \(t\in\numSelect_{a,f}(T)\), let's denote the loss due to pruning at time \(t\) by \(\ell_{a,f}^{(prune)}(t)\) and similarly let the loss due to pulling at time \(t\) by \(\ell_{a,f}^{(pull)}(t)\). Note that by construction, the loss due to the pruning operation is deterministic and zero. That is, for any \(t\in\numSelect_{a,f}(T)\),  \(\ell_{a,f}^{(prune)}(t)=0\) and \(\ell_{a,f}^{(pull)}(t)=1-2\isMatched_a(t)\). 
Furthermore, note that 
\begin{align*}
     V_{a,f}(T) &\leq \sum_{t\in\tau_{a,f}(T)} |\ell_{a,f}^{(pull)}(t)-\ell_{a,f}^{(pull)}(t-1)| \\ 
     &\underset{(a)}{\leq} 2\sum_{t\in\tau_{a,f}(T)} \one\lr{\collisionEvent_{a,f}(t-1),\collisionEvent_{a,f}^{\comp}(t)} + \one\lr{\collisionEvent_{a,f}^{\comp}(t-1),\collisionEvent_{a,f}(t)} \\ 
     &\leq 4 \min\lb{ \sum_{t=1}^{T}\one\lr{\collisionEvent_{a,f}^{\comp}(t)}, \sum_{t=1}^{T} \one\lr{\collisionEvent_{a,f}(t)} } \\ 
     &=  4 \min\lb{ \numPotMatches_{a,f}(T), \numPotCollide_{a,f}(T)}
\end{align*}
where the factor of 2 in 
is by the fact that a path length change in going from matching to potential collision or collision to potential matching is 2. The remaining inequalities follow from algebra. 

Furthermore, we have
\begin{align*}
    V_{a,f}(T) &= \sum_{t\in\numSelect_{a,f}(T)}\one\lr{\pullInstant_{a,f}(t)=1,\pullInstant_{a,f}(t-1)=1}|\ell_{a,f}^{(pull)}(t)-\ell_{a,f}^{(pull)}(t-1)|\\&\quad\quad\quad + \sum_{t\in\numSelect_{a,f}(T)} \one\lr{\pullInstant_{a,f}(t)=0,\pullInstant_{a,f}(t-1)=1}|\ell_{a,f}^{(pull)}(t)-\ell_{a,f}^{(pull)}(t-1)|
\\
    &\leq \sum_{t\in\numSelect_{a,f}(T)}\one\lr{\pullInstant_{a,f}(t)=1,\pullInstant_{a,f}(t-1)=1}|\ell_{a,f}^{(pull)}(t)-\ell_{a,f}^{(pull)}(t-1)|\\&\quad\quad\quad + 2\sum_{t\in\numSelect_{a,f}(T)} \one\lr{\pullInstant_{a,f}(t)=0,\pullInstant_{a,f}(t-1)=1} \\
    & = \sum_{t\in\numSelect_{a,f}(T)}\one\lr{\pullInstant_{a,f}(t)=1,\pullInstant_{a,f}(t-1)=1}|\ell_{a,f}^{(pull)}(t)-\ell_{a,f}^{(pull)}(t-1)|\\&\quad\quad\quad + 2\sum_{t=2}^{T} \one\lr{\pullInstant_{a,f}(t)=0,\pullInstant_{a,f}(t-1)=1} 
    \end{align*}
    \begin{align*}
    &= 2\sum_{t\in\numSelect_{a,f}(T)}\one\lr{\pullInstant_{a,f}(t)=1,\pullInstant_{a,f}(t-1)=1,\isMatched_{a}(t)=0,\isMatched_{a}(t-1)=1}\\ &\quad\quad\quad+2\sum_{t\in\numSelect_{a,f}(T)}\one\lr{\pullInstant_{a,f}(t)=1,\pullInstant_{a,f}(t-1)=1,\isMatched_{a}(t)=1,\isMatched_{a}(t-1)=0}\\&\quad\quad\quad +2 \sum_{t\in\numSelect_{a,f}(T)} \one\lr{\pullInstant_{a,f}(t)=0,\pullInstant_{a,f}(t-1)=1}  \end{align*}
    \begin{align*} 
    &\leq 2\lr{\sum_{t\in\numSelect_{a,f}(T)}\one\lr{ \pullInstant_{a,f}(t)=1,\isMatched_{a}(t)=0} + \one\lr{ \pullInstant_{a,f}(t-1)=1,\isMatched_{a}(t-1)=1}}\\&\quad\quad\quad +2 \sum_{t\in\numSelect_{a,f}(T)} \one\lr{\pullInstant_{a,f}(t)=0,\pullInstant_{a,f}(t-1)=1} \\ 
    &\leq 4\lr{\numMatches_{a,f}(T)+\numCollide_{a,f}(T)}
\end{align*}
\end{proof}

\subsection{Technical Lemma}
\begin{lemma}\label{lem: TwoArmBregman}
For any \(L\in \R^2\) and \(X\in \Delta(\R^2) \) the update \(X_{+} = \arg\min_{Z\in\Delta(\R^2)} \lara{Z,L}+D_{\psi}(Z,X)\) can be analytically solved to be \(X_{+} = [x_{+},1-x_+]\) where 
\begin{align}\label{eq: UpdateX}
x_{+} = \frac{2+\xi-\sqrt{4+\xi^2}}{2\xi}
\end{align}
where \(\xi = \eta(L_1-L_2)+\frac{1}{X_1}-\frac{1}{X_2}\). For better interpretation we provide the graph for update \eqref{eq: UpdateX} in the Figure \ref{fig:plot_}.
\end{lemma}
\begin{proof}
For any \(X,Z\in \Delta(\R^2)\) we represent \(X=[x,1-x]\) and \(Z=[z,1-z]\) for \(x,z\in[0,1]\). Under this notation we can write \(D_{\psi}(Z,X) = \frac{1}{\eta}\lr{     \log\lr{\frac{x}{z}}+\log\lr{\frac{1-x}{1-z}}+\frac{z-x}{x}+\frac{x-z}{1-x}}\). Thus the optimization problem becomes  
\begin{align*}
    x_{+} &= \arg\min_{z\in[0,1]} \lara{z,L}+D_{\psi}(z,X) \\ 
    &= \arg\min_{z\in[0,1]}zL_1+(1-z)L_2+\frac{1}{\eta}\lr{     \log\lr{\frac{x}{z}}+\log\lr{\frac{1-x}{1-z}}+\frac{z-x}{x}+\frac{x-z}{1-x}} \\ 
    &=  \arg\min_{z\in[0,1]}zL_1+(1-z)L_2+\frac{1}{\eta}\lr{     -\log\lr{z}-\log\lr{1-z}+\frac{z}{x}-\frac{z}{1-x}}
\end{align*}

Let \(f(z) = zL_1+(1-z)L_2+\frac{1}{\eta}\lr{     -\log\lr{z}-\log\lr{1-z}+\frac{z}{x}-\frac{z}{1-x}} \). Note that \(f(0)=+\infty,\) and \(f(1) = +\infty\) so the minimizer of \(f(z)\) lies stricly inside \([0,1]\). Therefore \(\nabla f(x_+) = 0\). We compute 
\begin{align*}
    \nabla f(z) = L_1-L_2+\frac{1}{\eta(1-z)}-\frac{1}{\eta z}+\frac{1}{\eta x }-\frac{1}{\eta(1-x)} = L_1-L_2+\frac{2z-1}{\eta z(1-z)}+\frac{1}{\eta x}-\frac{1}{\eta(1-x)}
\end{align*}

Imposing the condition \(\nabla f(x_+)= 0\) implies that 
\begin{align*}
 \xi x_+^2-(2+\xi)x_{+}+1=0
\end{align*}
where \(\xi = \eta(L_1-L_2)+\frac{1}{x}-\frac{1}{1-x}\). Thus there are two possibilities 
\begin{align*}
    x_+ = \frac{2+\xi+\sqrt{4+\xi^2}}{2\xi}, \quad \text{or}\quad x_+ = \frac{2+\xi-\sqrt{4+\xi^2}}{2\xi}, 
\end{align*}
However the first possibility implies that \(x_+ > 1\), thus the only solution which lies in \((0,1)\) is the latter. This completes the proof. 
\end{proof}

\begin{figure}
    \centering
  \begin{tikzpicture}
 
\begin{axis}[xmin = -20, xmax = 20,
    ymin = 0, ymax = 1.0,
    xtick distance = 4,
    ytick distance = 0.5,
    grid = both,
    minor tick num = 1,
    major grid style = {lightgray},
    minor grid style = {lightgray!25},
    width = \textwidth,
    height = 0.25\textwidth,
    xlabel = {$\forgetMeasure$},
    ylabel = {$x_+$},
    grid,]
    
    \addplot[
        domain = -20:20,
        samples = 200,
        smooth,
        thick,
        blue,
    ] {(2+x-sqrt(x*x+4))/(2*x)};
\end{axis}
 \end{tikzpicture}
    \caption{Update function of pulling probability based on line 10 in Algorithm \ref{alg:AdaptivePart}}
    \label{fig:plot_}
\end{figure}
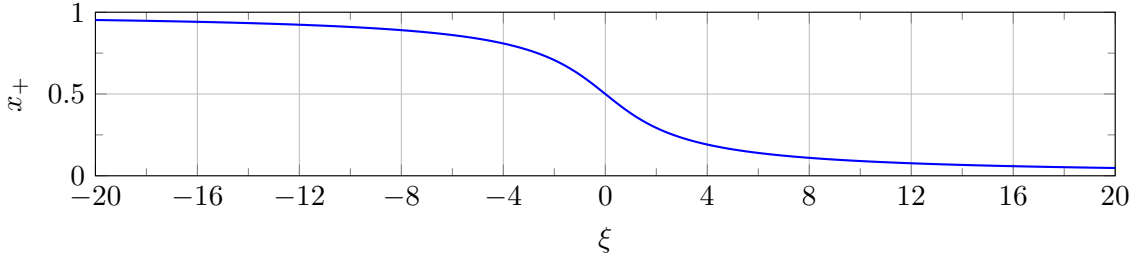

\section{Proofs of main Lemmas}
We introduce the following notation for every \(a\in\actSet,f\in\firmSet\)
\begin{align*}
    \collisionEvent_{a,f}(t)=\one\lr{\exists a'\in\actSet: \chosenFirm_{a'}(t)=f, \utilityFirm_{f}(a')>\utilityFirm_{f}(a)},
\end{align*}
which characterizes an event some agent more preferred than \(a\) by firm \(f\) has requested firm \(f\). We now present the proofs of Lemmas in main paper in the following subsections.
\subsection{Proof of Lemma \ref{lem: ChooseF}}
Proof of Lemma \ref{lem: ChooseF} follows directly from the following Lemma.
\begin{lemma}\label{lem: App_SubEvent}
    The event that agent \(a\) chooses a firm \(f\in\firmSet\) at time \(t\in[T]\) satisfies 
    \begin{align}\label{eq: ChosenFirmDecomp}
    \{\isMatched_a(t)=1,\chosenFirm_a(t)=f\} \subset \lb{\isMatched_a(t)=1,\UCB_{a,\stableArm_a}(t)\leq \UCB_{a,f}(t) }\bigcup \lb{\actChoose_{a,f}(t) = 1,\actChoose_{a,\stableArm_a}(t)=0}. 
    \end{align}
\end{lemma}
\begin{proof}
For any agent \(a\) fix some \(f\). Recall that \(\chosenFirm_a(t) = f\) implies that agent a has chosen to pull arm \(f\). Based on design of Algorithm \ref{alg:prunedUCBFinal} there are two possibilities: either all the firms with higher \UCB~ than firm \(f\) got pruned and the firm \(f\) was requested; or all of the firms in \(\firmSet\) got pruned and the firm \(f\) got selected as it was having highest \UCB.
Thus, 
\begin{align*}
&\lb{\chosenFirm_a(t) = f} = \lb{\actChoose_{a,f}(t) = 1} \bigcup \lb{\actChoose_{a,f}(t)=0~\forall~f\in\firmSet,\UCB_{a,f} \geq  \UCB_{a,f'} \ \forall \ f' \in \firmSet}
\\&\underset{(i)}{=} \lb{\actChoose_{a,f}(t) = 1,\UCB_{a,f_a^*}(t)\geq \UCB_{a,f}(t)}\bigcup \lb{\actChoose_{a,f}(t) = 1,\UCB_{a,f_a^*}(t)\leq \UCB_{a,f}(t)} \\ 
&\quad\quad \bigcup \lb{\actChoose_{a,f}(t)=0~\forall~f\in\firmSet,\UCB_{a,f} \geq \UCB_{a,f'} \ \forall \ f' \in \firmSet} \\ 
&\underset{(ii)}{\subset}\lb{\actChoose_{a,f}(t) = 1,\UCB_{a,f_a^*}(t)\geq \UCB_{a,f}(t)}\bigcup \lb{\actChoose_{a,f}(t) = 1,\UCB_{a,f_a^*}(t)\leq \UCB_{a,f}(t)} \\ 
&\quad\quad \bigcup \lb{\UCB_{a,\stableArm_a}(t)\leq \UCB_{a,f}(t) } \end{align*}
\begin{align*}
&\underset{(iii)}{\subset} \lb{\actChoose_{a,f}(t) = 1,\actChoose_{a,\stableArm_a}(t)=0,\UCB_{a,f_a^*}(t)\geq \UCB_{a,f}(t)}\\ 
&\quad\quad \bigcup \lb{\actChoose_{a,f}(t) = 1,\UCB_{a,f_a^*}(t)\leq \UCB_{a,f}(t)} \bigcup \lb{\UCB_{a,\stableArm_a}(t)\leq \UCB_{a,f}(t) } \\ 
&\underset{(iv)}{\subset}  \lb{\actChoose_{a,f}(t) = 1,\actChoose_{a,\stableArm_a}(t)=0,\UCB_{a,f_a^*}(t)\geq \UCB_{a,f}(t)} \bigcup \lb{\UCB_{a,\stableArm_a}(t)\leq \UCB_{a,f}(t) }\\
&\underset{(v)}{\subset}  \lb{\actChoose_{a,f}(t) = 1,\actChoose_{a,\stableArm_a}(t)=0} \bigcup \lb{\UCB_{a,\stableArm_a}(t)\leq \UCB_{a,f}(t) }
\end{align*}
where in \((i)\) we introduced two complementary events \(\{\UCB_{a,f_a^*}(t)\geq \UCB_{a,f}(t)\}\) and \(\{\UCB_{a,f_a^*}(t)\leq \UCB_{a,f}(t)\}\). Note that \((ii)\)
holds due to the fact that \(\{\UCB_{a,\chosenFirm_a(t)} \geq  \UCB_{a,f} \ \forall \ f \in \firmSet\}\) implies \(\{\UCB_{a,\chosenFirm_a(t)} \geq \UCB_{a,\stableArm_a}\}\). Furthermore, \((iii)\) holds due to the fact that a firm with lower UCB will be pulled only if all the firms with higher UCB are pruned. Finally, \((iv),(v)\) holds by dropping appropriate events. 

The result follows by noting that 
\begin{align*}
    &\one\lr{\isMatched_a(t)=1,\chosenFirm_a(t)=f} \\  &\subset \lr{\lb{\actChoose_{a,f}(t) = 1,\actChoose_{a,\stableArm_a}(t)=0} \bigcup \lb{\UCB_{a,\stableArm_a}(t)\leq \UCB_{a,f}(t) }}\bigcap\one\lr{\isMatched_a(t)=1} \\ 
    &\subset \lb{\isMatched_a(t)=1,\UCB_{a,\stableArm_a}(t)\leq \UCB_{a,f}(t) }\bigcup \lb{\actChoose_{a,f}(t) = 1,\actChoose_{a,\stableArm_a}(t)=0}
\end{align*}
\end{proof}
\begin{remark}\label{rem: TSUCB}
The results in Lemma \ref{lem: App_SubEvent} holds even if we replace UCB subroutine in Algorithm \ref{alg:prunedUCBFinal} with any other index based stochastic bandit subroutine, e.g. Thompson sampling. 
\end{remark}
\subsection{Proof of Lemma \ref{lem: MainLemma}}
We present the proof of each result \textbf{(L1)-(L5)} in Lemma \ref{lem: MainLemma} individually in the following subsubsections. Before that we define an important notation as follows:
\begin{align}
    \collisionEvent_{a,f}(t) = \one\lr{ \exists~a'\in \actSet: \chosenFirm_{a'}(t)=f, \utilityFirm_{f}(a')\geq \utilityFirm_{f}(a)}
\end{align}
\subsubsection{Proof of \textbf{(L1)} in Lemma \ref{lem: MainLemma}}
From \eqref{eq: RegretDefinition} we get 
\begin{align*}
     \sum_{i=1}^{k}\sum_{a\in\subActSet_i}R_{a} &\leq   \bar{\Delta}  \sum_{i=1}^{k}\sum_{a\in\subActSet_i}\sum_{f\in\subArm_a}\avg[\numMatches_{a,f}(T)] +   u \sum_{i=1}^{k}\sum_{a\in\subActSet_i}\sum_{f\in F\backslash \{f^*_a\}} \avg[C_{a,f}(T)]\\ &\hspace{3cm} + \bar{u}\sum_{i=1}^{k}\sum_{a\in\subActSet_i} \avg[\numCollide_{a,\stableArm_a}(T)], \\ 
    &\leq \bar{C}\bigg( \sum_{i=1}^{k}\sum_{a\in\subActSet_i}\sum_{f\in\subArm_a}\avg[\numMatches_{a,f}(T)] +    \sum_{i=1}^{k}\sum_{a\in\subActSet_i}\sum_{f\in F\backslash \{f^*_a\}} \avg[C_{a,f}(T)] \\&\hspace{3cm}+ \sum_{i=1}^{k}\sum_{a\in\subActSet_i} \avg[\sum_{t=1}^{T}H_{a,f^*_a}(t)]\bigg)
\end{align*}
where \(\bar{\gap} = \max_{a,f}\gap_{a}(f)\) and \(\bar{u} =\max_{a}\utilityAgent_a(\stableArm_a)\). This completes the proof
\subsubsection{Proof of \textbf{(L2)} in Lemma \ref{lem: MainLemma}}
Proof of \textbf{(L2)} in Lemma \ref{lem: MainLemma} follows immediately from the following more general result. 
\begin{lemma}\label{lem: numMatchingApp}
For any agent \(a\in\actSet\) using Algorithm \ref{alg:prunedUCBFinal} the expected number of matches with any set \(\tilde{\firmSet}\subseteq \subArm_a\) can be bounded as  
\begin{align*}
    \avg[\numMatches_{a,\tilde{\firmSet}}(T)] \leq \bigo\lr{|\tilde{\firmSet}|\lr{\log(T)+\frac{\log(T)}{\gap^2}}+\avg\ls{\sum_{t=1}^{T}\one\lr{\collisionEvent_{a,\stableArm_a}(t) }}}
\end{align*}
where \(\gap = \min_{a,f}\gap_{a}(f)\).
\end{lemma}
\begin{proof}
Note that we call an agent \(a\) matches with firm \(f\) at time \(t\) if \(\isMatched_a(t)=1\) and \(\chosenFirm_a(t) = f\). Therefore the total number of matchings between \(a\) and \(f\) till time \(T\) is \(\numMatches_{a,f}(T) = \sum_{t=1}^{T}\one\lr{\isMatched_a(t)=1,\chosenFirm_a(t)=f}\).  Therefore from Lemma \ref{lem: ChooseF} the following holds for every \(f\in\tilde{\firmSet}\):
\begin{align*}
&\numMatches_{a,\tilde{\firmSet}}(\horizon) = \sum_{f\in\tilde{\firmSet}}\sum_{t=1}^{\horizon} \one\lr{ \isMatched_a(t) = 1 , \chosenFirm_a(t) = f }\\
&\leq \sum_{f\in\tilde{\firmSet}}\sum_{t=1}^{\horizon}\lr{\one \lr{\isMatched_{a}(t)=1,\chosenFirm_a(t)=f,\UCB_{a,f}(t)\geq \UCB_{a,\stableArm_a}(t)}  +   \one\lr{\actChoose_{a,f}(t)=1,\actChoose_{a,\stableArm_a}=0}}\\
&\leq \sum_{f\in\tilde{\firmSet}}\sum_{t=1}^{\horizon}\one \lr{\isMatched_{a}(t)=1,\chosenFirm_a(t)=f,\UCB_{a,f}(t)\geq \UCB_{a,\stableArm_a}(t)} \\&\hspace{1cm} +  \sum_{t=1}^{\horizon}\sum_{f\in\tilde{\firmSet}} \one\lr{\actChoose_{a,f}(t)=1,\actChoose_{a,\stableArm_a}=0}
\end{align*}
\begin{align*}
&\leq\sum_{f\in\tilde{\firmSet}}\underbrace{\sum_{t=1}^{\horizon}\one \lr{\isMatched_{a}(t)=1,\chosenFirm_a(t)=f,\UCB_{a,f}(t)\geq \UCB_{a,\stableArm_a}(t)}}_{\text{Term A}}  +  \underbrace{\sum_{t=1}^{\horizon} \one\lr{\actChoose_{a,\stableArm_a}=0}}_{\text{Term B}}
\end{align*}

For any fixed firm \(f\in\tilde{\firmSet}\) we now bound Term A. For that purpose, 
    define an event \[\eventUCB_{a,f}(t) \defas \lb{ \UCB_{a,f}(t) \geq \utilityAgent_a({\stableArm_a})-\epsilon }  =  \left\{\mean_{a,f}(t-1)+\sqrt{\frac{2\log(B_a(t))}{\numMatches_{a,f}(t-1)}}\geq \utilityAgent_a({\stableArm_a})-\epsilon\right\},\] where 
    \(
    B_a(t)\Let 1+\totalMatch_a(t)\log^2\lr{\totalMatch_a(t)} \leq 1+t\log^2(t) \teL \bar{B}(t),   \)\footnote{ The inequality holds due to the fact that \(\totalMatch_a(t)\leq t\) and monotonicity of the mapping \(x\mapsto 1+x\log^2(x)\).}.

    Using this notation, we have
\begin{align*}
    \text{Term A}=& \underbrace{\sum_{t=1}^{\horizon}\one(\isMatched_{a}(t)=1,\chosenFirm_a(t)=f,\UCB_{a,f}(t)\geq \UCB_{a,\stableArm_a}(t),\eventUCB_{a,f}(t))}_{\text{Term C}}\\&\hspace{1cm}+\underbrace{\sum_{t=1}^{\horizon}\one(\isMatched_a(t)=1,\chosenFirm_a(t)=f,\UCB_{a,f}(t)\geq \UCB_{a,\stableArm_a}(t),\eventUCB_{a,f}^\comp(t))}_{\text{Term D}}
\end{align*}
We shall first bound \(\avg[\text{Term C}]\) below: 
\begin{align*}
    &\text{Term C} = \sum_{t=1}^{\horizon}\one(\isMatched_a(t)=1,\chosenFirm_a(t)=f,\UCB_{a,f}(t)\geq \UCB_{a,f_a^*}(t),\eventUCB_{a,f}(t)) \\
    &\leq \sum_{t=1}^{\horizon}\one(\isMatched_{a}(t)=1,\chosenFirm_a(t)=f,\eventUCB_{a,f}(t)) 
    \\ &=  \sum_{t=1}^{\horizon}\one\bigg({\isMatched_{a}(t)=1,\chosenFirm_{a}(t)=f,\mean_{a,f}(t-1)+\sqrt{\frac{2\log(B_a(t))}{\numMatches_{a,f}(t-1)}}\geq \utilityAgent_a(\stableArm_a)-\epsilon}\bigg)\\
    &\leq \sum_{t=1}^{\horizon}\one\bigg({\isMatched_{a}(t)=1,\chosenFirm_{a}(t)=f,\mean_{a,f}(t-1)+\sqrt{\frac{2\log(B_a(\horizon))}{\numMatches_{a,f}(t-1)}}\geq \utilityAgent_a(\stableArm_a)-\epsilon}\bigg)
    \\
    &= \sum_{t=1}^{\horizon}\sum_{s=0}^{t-1}\one\bigg({\isMatched_{a}(t)=1,\chosenFirm_{a}(t)=f,\mean_{a,f}^{(s)}+\sqrt{\frac{2\log(B_a(\horizon))}{s}}\geq \utilityAgent_a(\stableArm_a)-\epsilon,\numMatches_{a,f}(t-1)=s}\bigg)
    \\
    &\leq \sum_{s=0}^{\horizon-1}\sum_{t=s+1}^{\horizon} \one\bigg(\chosenFirm_{a}(t)=f,\mean_{a,f}^{(s)}+\sqrt{\frac{2\log(B_a(\horizon))}{s}}\geq \utilityAgent_a(\stableArm_a)-\epsilon,\numMatches_{a,f}(t-1)=s,\numMatches_{a,f}(t)=s+1\bigg)\\
    &\leq \sum_{s=0}^{\horizon-1} \one\bigg(\mean_{a,f}^{(s)}+\sqrt{\frac{2\log(B_a(\horizon))}{s}}\geq \utilityAgent_a(\stableArm_a)-\epsilon\bigg)\\
    &\leq \sum_{s=0}^{\horizon-1} \one\bigg(\mean_{a,f}^{(s)}-\utilityAgent_{a}(f)+\sqrt{\frac{2\log({\bar{B}(\horizon)})}{s}}\geq \underbrace{\utilityAgent_a(\stableArm_a)-\utilityAgent_a(f)}_{\gap_{a}(f)}-\epsilon\bigg),
\end{align*}
where \(\mu_{a,f}^{(s)}\) is defined to be the empirical utility that agent \(a\) obtains on \(s\) independent successful pulls of arm \(f\). 
Using Lemma \ref{Lem: LemmaLattimore} to further bound \(\avg[\text{Term C}]\) we get 
\[
\avg[\text{Term C}]\leq 1+\frac{2}{(\gap_{a}(f)-\epsilon)^2}\lr{\log(\bar{B}(T)+\sqrt{\pi\log(\bar{B}(\horizon))}+1)} 
\]
Next, we bound \(\avg[\text{Term D}]\) below: 
\begin{align*}
    \avg[\text{Term D}] &=\avg\ls{\sum_{t=1}^{\horizon}\one(\isMatched_a(t)=1,\chosenFirm_a(t)=f,\UCB_{a,f}(t)\geq \UCB_{a,\stableArm_a}(t),\UCB_{a,f}(t)\leq \utilityAgent_a(\stableArm_a)-\epsilon}\\
    &\leq \avg\ls{\sum_{t=1}^{\horizon}\one\lr{\isMatched_{a}(t)=1,\mean_{a,\stableArm_a}(t-1)+\sqrt{\frac{2\log(B_a(t))}{\numMatches_{a,\stableArm_a}(t-1)}}\leq \utilityAgent_a({\stableArm_a})-\epsilon}} \\ 
    & {\leq } \sum_{t=1}^{\horizon}\sum_{s=0}^{\horizon-1} \Pr{\mean_{a,\stableArm_a}^{(s)}+\sqrt{\frac{2\log(\bar{B}(t))}{s}}\leq \utilityAgent_a(\stableArm_a)-\epsilon} \\ 
    &{\leq } \sum_{t=1}^{\horizon}\sum_{s=0}^{\horizon-1} \exp\lr{-\frac{s\lr{\sqrt{\frac{2\log(\bar{B}(t))}{s}}+\epsilon}^2}{2}} \\
    &{\leq } \sum_{t=1}^{\horizon}\frac{1}{\bar{B}(t)}\sum_{s=1}^{\horizon} \exp\lr{-\frac{s\epsilon^2}{2}} \\
    &\leq \frac{\epsilon^2}{2}\sum_{t=0}^{\horizon-1}\frac{1}{\bar{B}(t)}
    \end{align*}
     which can further be bounded as \(\avg[\text{Term D}]\leq \frac{5}{\epsilon^2}\) in \cite[Exercise 8.1]{lattimore2020bandit}.
For simplicity we choose \(\epsilon = \gap_{a}(f)/2\) which ensures that 
\(
\avg[\text{Term A}] \leq \bigo\lr{\frac{\log(T)}{\lr{\gap_{a}(f)}^2}}
\)

Now let's turn our attention to Term B which characterizes the number of times agent \(a\) has pruned the stable match.  
    Using Lemma \ref{lem: BoundingTermA} we have 
     \begin{align*}
         \avg[\text{Term B}] \leq \bigo\lr{ \avg\ls{\sum_{t=1}^{T}\one\lr{\collisionEvent_{a,\stableArm_a}(t) }}+ \bigo(\log(T))}
     \end{align*}

Thus the Term A is bounded by number of there can be potential collisions at the stable firm. This concludes the proof of this lemma.
\end{proof}

\subsubsection{Proof of \textbf{(L3)} in Lemma \ref{lem: MainLemma}}
In this part, we prove a result which is more general than \textbf{(L3)} in Lemma \ref{lem: MainLemma}.  
\begin{lemma}\label{lem: NumCollisionLemma}
Expected number of collisions faced by agent \(a\) on the set of firms \(\firmSet^\dagger\subseteq\firmSet \backslash\{\stableArm_a\}\)
{
\begin{align}\label{eq: AppCollisionGeneral}
    \sum_{f\in\firmSet^\dagger}\avg[\numCollide_{a,f}(T)] \leq \bigo\lr{|\firmSet^\dagger|\log(T)+\avg[\numMatches_{a,\underline{{\firmSet}}^\dagger_a}(T)]+\avg[\numMatches_{a,\bar{{\firmSet}}^\dagger_a}(T)]+ \avg\ls{\sum_{t=1}^{T}\one\lr{\collisionEvent_{a,\stableArm_a}(t) }}},
\end{align}
where \(\underline{\firmSet}^\dagger_a = \subArm_a\cap \firmSet^\dagger\) and \(\bar{\firmSet}^\dagger_a=\superArm_a\cap \firmSet^\dagger\). Additionally 
\begin{align}\label{eq: AppCollision}
    \avg\ls{\numCollide_{a,\stableArm_a}(T)} \leq  \avg\ls{\sum_{t=1}^{T}\one\lr{ \collisionEvent_{a,\stableArm_a}(t) }}
\end{align}
}
\end{lemma}
\begin{proof}
To compute the number of collisions, we compute the following for \(a\in\actSet\) and \(f\in\firmSet\backslash\{\stableArm_a\}\) 
\begin{align*}
    &\sum_{f\in\firmSet^\dagger}\numCollide_{a,f}(\horizon) = \sum_{f\in\firmSet^\dagger}\sum_{t=1}^{\horizon} \one\lr{\chosenFirm_a(t)=f,\collisionEvent_{a,f}(t)}\\
    &= \sum_{f\in\firmSet^\dagger}\sum_{t=1}^{\horizon} \one\lr{\actChoose_{a,f}(t)=1,\actQuery_{a,f}(t)=1,\collisionEvent_{a,f}(t)} \\&\quad +\sum_{f\in\firmSet^\dagger} \sum_{t=1}^{\horizon}\one\lr{ \actChoose_{a,f'}(t) = 0 \ \forall \ f'\in\firmSet, f_a(t)=f,\collisionEvent_{a,f}(t) }\\
    &\leq\sum_{f\in\firmSet^\dagger} \sum_{t=1}^{\horizon} \one\lr{\actChoose_{a,f}(t)=1,\actQuery_{a,f}(t)=1,\collisionEvent_{a,f}(t)} + \sum_{f\in\firmSet^\dagger}\sum_{t=1}^{\horizon}\one\lr{ \actChoose_{a,\stableArm_{a}}(t) = 0 ,f_a(t)=f},\\
    &\leq\sum_{f\in\firmSet^\dagger} \sum_{t=1}^{\horizon} \one\lr{\actChoose_{a,f}(t)=1,\actQuery_{a,f}(t)=1,\collisionEvent_{a,f}(t)} + \sum_{t=1}^{\horizon}\one\lr{ \actChoose_{a,\stableArm_{a}}(t) = 0 },
\end{align*}
where the first inequality holds because \(\{\actChoose_{a,f'}(t) = 0 \ \forall \ f'\in\firmSet\}\) implies that \(\{\actChoose_{a,\stableArm_{a}}(t) = 0 \}\).
Using \eqref{eq: Implication1} we have: for all \(a\in\actSet,f\in\firmSet\) and \(\varpi\in(0,32\eta)\subset (0,1)\)
{
\begin{align*}
   &\sum_{f\in\firmSet^\dagger} \avg[\numCollide_{a,f}(T)]\\& \leq \sum_{f\in\firmSet^\dagger}\lr{(1+\varpi)\avg[\numMatches_{a,f}(T)] +\bigo(\log(T)) + \varpi\avg[\numCollide_{a,f}(T)] + \avg\ls{\sum_{t=1}^{\horizon}{\one\lr{ \actChoose_{a,\stableArm_{a}} = 0 }}}} \\ 
    &\leq \bigo\lr{ |\firmSet^\dagger|\log(T)+\sum_{f\in\firmSet^\dagger}\avg[\numMatches_{a,f}(T)]}+ \avg\ls{\sum_{t=1}^{T}\one\lr{\collisionEvent_{a,\stableArm_a}(t) }} + \varpi\sum_{f\in\firmSet^\dagger}\avg[\numCollide_{a,f}(T)]
\end{align*}
where the last inequality is due to Lemma \ref{lem: BoundingTermA}. In summary,
\begin{align*}
    \sum_{f\in\firmSet^\dagger}\avg[\numCollide_{a,f}(T)] &\leq \bigo\lr{|\firmSet|\bigo(\log(T))+\sum_{f\in\firmSet^\dagger}\lr{\avg[\numMatches_{a,f}(T)]}}+ \avg\ls{\sum_{t=1}^{T}\one\lr{\collisionEvent_{a,\stableArm_a}(t) }} \\ 
    &\leq \bigo\lr{|\firmSet^\dagger|\log(T)+\avg[\numMatches_{a,\underline{{\firmSet}}^\dagger_a}(T)]+\avg[\numMatches_{a,\bar{{\firmSet}}^\dagger_a}(T)]+ \avg\ls{\sum_{t=1}^{T}\one\lr{\collisionEvent_{a,\stableArm_a}(t) }}}
\end{align*}
This completes the proof of \eqref{eq: AppCollisionGeneral}.
We now prove \eqref{eq: AppCollision}. We note that 
\begin{align*}
    \avg\ls{\numCollide_{a,\stableArm_a}(T)} &= \avg\ls{\sum_{t=1}^{T}\one\lr{ \chosenFirm_a(t)=f,\collisionEvent_{a,\stableArm_a}(t)}}\leq \avg\ls{ \sum_{t=1}^{T}\one\lr{\collisionEvent_{a,\stableArm_a}(t)}}.
\end{align*}
This completes the proof.
}
\end{proof}

\subsubsection{Proof of \textbf{(L4)} in Lemma \ref{lem: MainLemma}}
We restate \textbf{(L4)} from Lemma \ref{lem: MainLemma} below:
\begin{lemma}\label{lem: App_CollisionStable}
For any \(i\in [\numMarket]\) we have 
\begin{align*}
    \sum_{j=1}^{i}\sum_{a\in\subActSet_j}\avg\ls{ \sum_{t=1}^{T}\one\lr{\collisionEvent_{a,\stableArm_a}(t)} }=  \bigo\lr{ C_i|\firmSet|\lr{\sum_{j=1}^{i}|\subActSet_j|}\log(T)\lr{1+\frac{1}{\gap^2} }},
\end{align*}
where \(C_i\) is a constant dependent on market \(\market_i\) such that \(C_1<C_2<...<C_{\numMarket}\).
\end{lemma}
\begin{proof}
For any \(k\in[\numMarket]\) define \(S_k = \sum_{i=1}^{k}\sum_{a\in\subActSet_i}\avg[\sum_{t=1}^{T}\one\lr{ \collisionEvent_{a,\stableArm_a}(t)}]\) and \(Z(T,\Delta) = |\firmSet|\log(T)\lr{1+\frac{1}{\Delta^2}}\).
Define \(f(\theta;\ell) = \sum_{j=1}^{\ell}\theta^j\), \(f(\theta;0)=1\) and \(g(\theta;\ell) = \sum_{j=0}^{\ell-1}\theta^j\). Moreover, let \(\mathcal{H}_i=\sum_{a\in\subActSet_i}\avg[\sum_{t=1}^{T}\one\lr{ \collisionEvent_{a,\stableArm_a}(t)}]\). Consequently \(S_k = \sum_{i=1}^{k}\mathcal{H}_i\).
We claim that

\begin{align}\label{eq: InductionH}
    S_{K} &\leq S_{K-\ell} + f(\theta;\ell)\mathcal{H}_{K-\ell}+ \sum_{p=1}^{\ell}g(\theta;p)\sum_{a\in\subActSet_{\numMarket-p+1}}\sum_{a'\in \cup_{j=1}^{\numMarket-\ell-1}\subActSet_{j}} \avg\ls{\numMatches_{a',\stableArm_a}(T)}\notag \\&\hspace{1cm}+Z(T,\Delta)\sum_{r=1}^{\ell} f(\theta;r)|\subActSet_{K-r}|
\end{align}

We prove this via induction. We first show that this holds for \(\ell=1\). Indeed note that 
\begin{align*}
    &S_{\numMarket} =S_{\numMarket-1} + \mathcal{H}_{K}= S_{\numMarket-1} + \sum_{a\in\subActSet_{\numMarket}} \avg\ls{\sum_{t=1}^{T}\one{\lr{ \collisionEvent_{a,\stableArm_a}(t)}}} \\ 
    &\underset{(a)}{\leq} S_{\numMarket-1} + \sum_{a\in\subActSet_{\numMarket}}\sum_{a'\in \cup_{j=1}^{\numMarket-2}\subActSet_{j}} \avg\ls{\numMatches_{a',\stableArm_a}(T)}+ \sum_{a\in\subActSet_{\numMarket}}\sum_{a'\in \subActSet_{\numMarket-1}} \avg\ls{\numMatches_{a',\stableArm_a}(T)} \\
     &\underset{(b)}{=}  S_{\numMarket-1} + \sum_{a\in\subActSet_{\numMarket}}\sum_{a'\in \cup_{j=1}^{\numMarket-2}\subActSet_{j}} \avg\ls{\numMatches_{a',\stableArm_a}(T)}+ \sum_{a'\in \subActSet_{\numMarket-1}}\sum_{f\in\subFirmSet_{\numMarket}} \avg\ls{\numMatches_{a',f}(T)}\\
     &\underset{(c)}{\leq}  S_{\numMarket-1} + \sum_{a\in\subActSet_{\numMarket}}\sum_{a'\in \cup_{j=1}^{\numMarket-2}\subActSet_{j}} \avg\ls{\numMatches_{a',\stableArm_a}(T)}+ \sum_{a'\in \subActSet_{\numMarket-1}}\avg\ls{\numMatches_{a',\subArm_{a'}}(T)}\\
     &\underset{(d)}{\leq}  S_{\numMarket-1} + \theta\sum_{a'\in\subActSet_{\numMarket-1}}\avg\ls{\sum_{t=1}^{T}\one{\lr{ \collisionEvent_{a',\stableArm_{a'}}(t)}}} + \sum_{a\in\subActSet_{\numMarket}}\sum_{a'\in \cup_{j=1}^{\numMarket-2}\subActSet_{j}} \avg\ls{\numMatches_{a',\stableArm_a}(T)}+ \theta|\subActSet_{\numMarket-1}|Z(T,\Delta)\\
    & {=}  S_{\numMarket-1} + \theta\mathcal{H}_{\numMarket-1} + \sum_{a\in\subActSet_{\numMarket}}\sum_{a'\in \cup_{j=1}^{\numMarket-2}\subActSet_{j}} \avg\ls{\numMatches_{a',\stableArm_a}(T)}+ \theta|\subActSet_{\numMarket-1}|Z(T,\Delta) 
\end{align*}
where the (a) holds due to \(\alpha-\)reducible structure which says that any agent in \(\subActSet_{K}\) will only get collided at stable arm if some agent from \(\cup_{j=1}^{k-1}\subActSet_j\) has also requested the stable firm. Next, \((b)\) holds due to the fact that for any agent \(a\in\subActSet_k\), the corresponding stable match \(\stableArm_a\in\subFirmSet_k\)(see Remark \ref{rem: MarketDecomp}). Next, (c) follows because for agents in \(\subActSet_{K-1}\), the set of suboptimal firms is super set of \(\firmSet_K\). This is again a property of \(\alpha-\)reducible structure. Finally \((d)\) follows from \textbf{(L2)} in Lemma \ref{lem: MainLemma} where \(\theta\) is the corresponding constant from big-oh notation.

Suppose the bound in \eqref{eq: InductionH} holds for \(\ell=L\) for some integer \(\ell\in\{2,3,...,K\}\). Then we show it also holds for \(\ell+1\). That is, 
\begin{align*}
    &S_{K} \underset{(a)}{\leq} S_{K-\ell} + f(\theta;\ell)\mathcal{H}_{K-\ell}+ \sum_{p=1}^{\ell}g(\theta;p)\sum_{a\in\subActSet_{\numMarket-p+1}}\sum_{a'\in \cup_{j=1}^{\numMarket-\ell-1}\subActSet_{j}} \avg\ls{\numMatches_{a',\stableArm_a}(T)} \\&\hspace{1cm}+Z(T,\Delta)\sum_{r=1}^{\ell} f(\theta;r)|\subActSet_{K-r}| \\        &\underset{(b)}{=} S_{K-\ell-1} + g(\theta;\ell+1)\mathcal{H}_{K-\ell}+ \sum_{p=1}^{\ell}g(\theta;p)\sum_{a\in\subActSet_{\numMarket-p+1}}\sum_{a'\in \cup_{j=1}^{\numMarket-\ell-1}\subActSet_{j}} \avg\ls{\numMatches_{a',\stableArm_a}(T)} \\&\hspace{1cm}+Z(T,\Delta)\sum_{r=1}^{\ell} f(\theta;r)|\subActSet_{K-r}| 
        \\
    &\underset{(c)}{\leq} S_{K-\ell-1} + g(\theta;\ell+1)\lr{\mathcal{H}_{K-\ell} +\sum_{p=1}^{\ell}\sum_{a\in\subActSet_{\numMarket-p+1}}\sum_{a'\in \subActSet_{\numMarket-\ell-1}} \avg\ls{\numMatches_{a',\stableArm_a}(T)} }\\&\hspace{1cm}+ \sum_{p=1}^{\ell}g(\theta;p)\sum_{a\in\subActSet_{\numMarket-p+1}}\sum_{a'\in \cup_{j=1}^{\numMarket-\ell-2}\subActSet_{j}} \avg\ls{\numMatches_{a',\stableArm_a}(T)} +Z(T,\Delta)\sum_{r=1}^{\ell} f(\theta;r)|\subActSet_{K-r}| \end{align*}
    \begin{align*}
    &\underset{(d)}{\leq} S_{K-\ell-1} + g(\theta;\ell+1)\lr{\sum_{p=1}^{K-\ell-1}\sum_{a'\in \subActSet_{p}}\sum_{a\in\subActSet_{K-\ell}}\avg[\numMatches_{a',\stableArm_a}] +\sum_{p=1}^{\ell}\sum_{a\in\subActSet_{\numMarket-p+1}}\sum_{a'\in \subActSet_{\numMarket-\ell-1}} \avg\ls{\numMatches_{a',\stableArm_a}(T)} }\\&\hspace{1cm}+ \sum_{p=1}^{\ell}g(\theta;p)\sum_{a\in\subActSet_{\numMarket-p+1}}\sum_{a'\in \cup_{j=1}^{\numMarket-\ell-2}\subActSet_{j}} \avg\ls{\numMatches_{a',\stableArm_a}(T)} +Z(T,\Delta)\sum_{r=1}^{\ell} f(\theta;r)|\subActSet_{K-r}| 
     \\
    &\underset{(e)}{=} S_{K-\ell-1} + g(\theta;\ell+1)\lr{\sum_{p=1}^{K-\ell-2}\sum_{a'\in \subActSet_{p}}\sum_{a\in\subActSet_{K-\ell}}\avg[\numMatches_{a',\stableArm_a}]
    +\sum_{p=1}^{\ell+1}\sum_{a\in\subActSet_{\numMarket-p+1}}\sum_{a'\in \subActSet_{\numMarket-\ell-1}} \avg\ls{\numMatches_{a',\stableArm_a}(T)} }\\&\hspace{1cm}+ \sum_{p=1}^{\ell}g(\theta;p)\sum_{a\in\subActSet_{\numMarket-p+1}}\sum_{a'\in \cup_{j=1}^{\numMarket-\ell-2}\subActSet_{j}} \avg\ls{\numMatches_{a',\stableArm_a}(T)} +Z(T,\Delta)\sum_{r=1}^{\ell} f(\theta;r)|\subActSet_{K-r}|\\
    &\underset{(f)}{\leq}  S_{K-\ell-1} + g(\theta;\ell+1)\lr{\sum_{p=1}^{K-\ell-2}\sum_{a'\in \subActSet_{p}}\sum_{a\in\subActSet_{K-\ell}}\avg[\numMatches_{a',\stableArm_a}]
    +\sum_{a'\in \subActSet_{\numMarket-\ell-1}} \avg\ls{\numMatches_{a',\subArm_{a'}}(T)} }\\&\hspace{1cm}+ \sum_{p=1}^{\ell}g(\theta;p)\sum_{a\in\subActSet_{\numMarket-p+1}}\sum_{a'\in \cup_{j=1}^{\numMarket-\ell-2}\subActSet_{j}} \avg\ls{\numMatches_{a',\stableArm_a}(T)} +Z(T,\Delta)\sum_{r=1}^{\ell} f(\theta;r)|\subActSet_{K-r}| \\
    &\underset{(g)}{=} S_{K-\ell-1} + g(\theta;\ell+1)\lr{
    \sum_{a'\in \subActSet_{\numMarket-\ell-1}} \avg\ls{\numMatches_{a',\subArm_{a'}}(T)} }\\&\hspace{1cm}+ \sum_{p=1}^{\ell+1}g(\theta;p)\sum_{a\in\subActSet_{\numMarket-p+1}}\sum_{a'\in \cup_{j=1}^{\numMarket-\ell-2}\subActSet_{j}} \avg\ls{\numMatches_{a',\stableArm_a}(T)} +Z(T,\Delta)\sum_{r=1}^{\ell} f(\theta;r)|\subActSet_{K-r}| \\
    &\underset{(h)}{\leq} S_{K-\ell-1} + g(\theta;\ell+1)\lr{
    \theta |\firmSet|Z(T,\gap)|\subActSet_{K-\ell-1}|+\theta \mathcal{H}_{K-\ell-1}}\\&\hspace{1cm}+ \sum_{p=1}^{\ell+1}g(\theta;p)\sum_{a\in\subActSet_{\numMarket-p+1}}\sum_{a'\in \cup_{j=1}^{\numMarket-\ell-2}\subActSet_{j}} \avg\ls{\numMatches_{a',\stableArm_a}(T)} +Z(T,\Delta)\sum_{r=1}^{\ell} f(\theta;r)|\subActSet_{K-r}| \end{align*}
     \begin{align*}
    &\underset{(i)}{=}S_{K-\ell-1} + f(\theta;\ell+1) \mathcal{H}_{K-\ell-1}+ \sum_{p=1}^{\ell+1}g(\theta;p)\sum_{a\in\subActSet_{\numMarket-p+1}}\sum_{a'\in \cup_{j=1}^{\numMarket-\ell-2}\subActSet_{j}} \avg\ls{\numMatches_{a',\stableArm_a}(T)} \\&\quad \quad \quad +Z(T,\Delta)\sum_{r=1}^{\ell+1} f(\theta;r)|\subActSet_{K-r}| 
\end{align*}
where \((a)\) holds by induction hypothesis, \((b)\) holds by definition of \(S_k\) and \(f(\theta;\ell),g(\theta;\ell)\), \((c)\) holds by moving some terms around and noting that \(g(\theta;\cdot)\) is increasing. Next, \((d)\) holds by \(\alpha-\)reducbility and definition of \(\mathcal{H}_k\) (same analysis as in base case of induction). Next, \((e)\) holds by splitting the terms. Next, \((f)\) holds by \(\alpha-\)reducilibility definition. Next \((g)\) holds by combining similar terms. Next \((h)\) holds by \textbf{(L2)} in Lemma \ref{lem: MainLemma}. Next, \((i)\) holds due to combining similar terms.

Thus we conclude that induction claim \eqref{eq: InductionH} holds true. We know that \(S_1=0\) therefore from \eqref{eq: InductionH} we obtain

\begin{align}\label{eq: InductionS}
S_k &\leq Z(T,\Delta)\sum_{r=1}^{K-1} f(\theta;r)|\subActSet_{K-r}| \leq 
\lr{\sum_{j=1}^{K-1}|{\subActSet}_j|}K\theta^{K-1}Z(T,\Delta).
\end{align} 
The term \(C_k=k\theta^{k-1}\) in the statement. This completes the proof.

\end{proof}
\subsubsection{Proof of (\textbf{L5}) in Lemma \ref{lem: MainLemma}}

So only thing to bound is matching with superoptimal firms. 
\begin{lemma}
 For any \(k\in[\numMarket]\) we have 
\begin{align*}
    \sum_{j=1}^{k}\sum_{a\in\subActSet_j}\sum_{f\in\superArm_a}\avg[\numMatches_{a,f}(T)] \leq \bigo\lr{C_i \lr{\sum_{j=1}^{k-1}|\subActSet_j|}|\firmSet|\log(T)\lr{1+\frac{1}{\gap^2}} },
\end{align*}
where \(C_i\) is a constant dependent on market \(\market_i\) such that \(C_1<C_2<...<C_{\numMarket}\).
\end{lemma}
\begin{proof}
For any \(k\in[\numMarket]\), define \(\tilde{S}_k = \sum_{i=1}^{k}\sum_{a\in \subActSet_i}\avg[M_{a,\superArm_a}(T)]\) and \(Z(T,\gap) = |F|\log(T)\lr{1+1/\gap^2}\). Define \(f(\theta;\ell) = \sum_{j=1}^{\ell}\theta^j\), \(f(\theta;0)=1\) and \(g(\theta;\ell) = \sum_{j=0}^{\ell-1}\theta^j\). Let \(\mathcal{H}_i\) \(=\sum_{a\in\subActSet_i}\avg[\sum_{t=1}^{T}\one\lr{ \collisionEvent_{a,\stableArm_a}(t)}]\) and \(\mathbb{M}_i=\sum_{a\in \subActSet_i}\avg[M_{a,\superArm_a}(T)]\) then \(\ts_k = \sum_{i=1}^{k}\mathbb{M}_i\). We claim that 
\begin{align}\label{eq: InductionTildeS}
    \tilde{S}_k \leq \bigo\lr{\tilde{\theta}^{k-1} \lr{\sum_{j=1}^{k-1}|\subActSet_j|}|\firmSet|Z(T,\gap)}
\end{align}
where \(\tilde{\theta}\) is a constant greater than 1. Note that the bound holds for \(k=1\) as there is not super-optimal firms for those agents. Let \eqref{eq: InductionTildeS} holds till some integer \(K-1\) then we show that it holds for \(K\) as well. Indeed,

We claim that 
\begin{align}\label{eq: inductionTS}
    \ts_{K}&\leq \ts_{K-\ell} + f(\ttheta;\ell)\mathbb{M}_{K-\ell}+ \sum_{p=1}^{\ell} g(\ttheta;p)\sum_{a\in\subActSet_{K-p+1}}\sum_{f\in \cup_{j\leq K-\ell-1 }\subFirmSet_{j}} \avg\ls{\numMatches_{a,f}} +  \sum_{p=1}^{\ell}f(\ttheta,p)\mathcal{H}_{K-p}\notag\\&\hspace{1cm}+Z(T,\Delta)\sum_{p=1}^{\ell} f(\ttheta,p) |\subActSet_{K-p}|
\end{align}

We prove \eqref{eq: InductionTildeS} by induction. First, consider the case \(\ell=1\)
\begin{align*}
    \ts_{\numMarket}&=\sum_{i=1}^{\numMarket}\sum_{a\in \subActSet_i}\avg[M_{a,\superArm_a}(T)] \\ 
    &\underset{(a)}{=} \ts_{\numMarket-1}+ \sum_{a\in\subActSet_{\numMarket}} \avg[M_{a,\superArm_a}(T)] \\ 
    &\underset{(b)}{\leq} \ts_{\numMarket-1}+\sum_{a\in\subActSet_{\numMarket}} \sum_{f\in \cup_{j\leq \numMarket-2}\subFirmSet_{j}}\avg[M_{a,f}(T)]+ \sum_{a\in\subActSet_{\numMarket}} \sum_{f\in \subFirmSet_{\numMarket-1}}\avg[M_{a,f}(T)]\\
     &\underset{(c)}{=} \ts_{\numMarket-1}+\sum_{a\in\subActSet_{\numMarket}} \sum_{f\in \cup_{j\leq \numMarket-2}\subFirmSet_{j}}\avg[M_{a,f}(T)]+  \sum_{a'\in \subActSet_{\numMarket-1}}\sum_{a\in\subActSet_{\numMarket}}\avg[M_{a,\stableArm_{a'}}(T)] \\ 
     &\underset{(d)}{\leq} \ts_{\numMarket-1}+\ttheta\sum_{a'\in \subActSet_{\numMarket-1}}\avg[M_{a',\superArm_{a'}}(T)] +\sum_{a\in\subActSet_{\numMarket}} \sum_{a'\in \cup_{j\leq \numMarket-2}\subActSet_{j}}\avg[M_{a,\stableArm_{a'}}(T)]\\&\hspace{1cm}+ \sum_{a'\in \subActSet_{\numMarket-1}}\ttheta{\lr{H_{a',f^*_{a'}}+Z(T,\Delta)}}\\ 
     &\underset{(e)}{=} \ts_{\numMarket-1}+\ttheta\mathbb{M}_{K-1} +\sum_{a\in\subActSet_{\numMarket}} \sum_{f\in \cup_{j\leq \numMarket-2}\subFirmSet_{j}}\avg[M_{a,f}(T)]+ \ttheta\mathcal{H}_{K-1}+Z(T,\gap)\ttheta|\subActSet_{K-1}| 
\end{align*}
where \((a)\) holds by definition, \((b)\) holds by using \(\alpha-\)reducilbe structure which ensures that set of superoptimal firms of any agent will lie in markets before it. Next, \((c)\) holds by property of alpha-reducible markets which ensures that for firm \(f\in \firmSet_{K-1}\) there exists agent \(a'\in\subActSet_{K-1}\) such that \(f=\stableArm_{a'}\). Next, \((d)\) holds by Lemma \ref{lem: BoundOtherMatch}. Next \((e)\) holds by rearrangement of terms. Next, we show that if \eqref{eq: InductionTildeS} holds for some \(\ell\) then it holds for \(\ell+1\) as well. That is, 
\begin{align*}
    &\ts_{K}\underset{(a)}{\leq} \ts_{K-\ell} + f(\ttheta;\ell)\mathbb{M}_{K-\ell}+ \sum_{p=1}^{\ell} g(\ttheta;p)\sum_{a\in\subActSet_{K-p+1}}\sum_{f\in \cup_{j\leq K-\ell-1} \subFirmSet_{j}} \avg\ls{\numMatches_{a,f}} +  \sum_{p=1}^{\ell}f(\ttheta,p)\mathcal{H}_{K-p}\notag\\&\hspace{1cm}+Z(T,\Delta)\sum_{p=1}^{\ell} f(\ttheta,p) |\subActSet_{K-p}| \\
    &\underset{(b)}{=}\ts_{K-\ell-1} + g(\ttheta;\ell+1)\mathbb{M}_{K-\ell}+ \sum_{p=1}^{\ell} g(\ttheta;p)\sum_{a\in\subActSet_{K-p+1}}\sum_{f\in \cup_{j\leq K-\ell-1} \subFirmSet_{j}} \avg\ls{\numMatches_{a,f}} +  \sum_{p=1}^{\ell}f(\ttheta,p)\mathcal{H}_{K-p}\notag\\&\hspace{1cm}+Z(T,\Delta)\sum_{p=1}^{\ell} f(\ttheta,p) |\subActSet_{K-p}|  \\
    &\underset{(c)}{=}\ts_{K-\ell-1} + g(\ttheta;\ell+1)\lr{ \sum_{a\in\subActSet_{K-\ell}} \sum_{f\in\cup_{j\leq K-\ell-2}\subFirmSet_j}  \avg\ls{\numMatches_{a,f}} +\sum_{a\in\subActSet_{K-\ell} } \sum_{f\in\firmSet_{K-\ell-1} }\avg\ls{\numMatches_{a,f}(T)} }\\&\hspace{1cm}+ \sum_{p=1}^{\ell} g(\ttheta;p)\sum_{a\in\subActSet_{K-p+1}}\sum_{f\in \cup_{j\leq K-\ell-1} \subFirmSet_{j}} \avg\ls{\numMatches_{a,f}} +  \sum_{p=1}^{\ell}f(\ttheta,p)\mathcal{H}_{K-p}\\&\hspace{1cm}+Z(T,\Delta)\sum_{p=1}^{\ell} f(\ttheta,p) |\subActSet_{K-p}|  \end{align*}
    \begin{align*}
    &\underset{(d)}{\leq} \ts_{K-\ell-1} + g(\ttheta;\ell+1)\lr{ \sum_{p=1}^{\ell+1}\sum_{a\in\subActSet_{K-p+1} } \sum_{f\in\firmSet_{K-\ell-1} }\avg\ls{\numMatches_{a,f}(T)} }\\&\hspace{1cm}+ \sum_{p=1}^{\ell+1} g(\ttheta;p)\sum_{a\in\subActSet_{K-p+1}}\sum_{f\in \cup_{j\leq K-\ell-2} \subFirmSet_{j}} \avg\ls{\numMatches_{a,f}} +  \sum_{p=1}^{\ell}f(\ttheta,p)\mathcal{H}_{K-p}\\&\hspace{1cm}+Z(T,\Delta)\sum_{p=1}^{\ell} f(\ttheta,p) |\subActSet_{K-p}| \\
    &\underset{(e)}{=} \ts_{K-\ell-1} + g(\ttheta;\ell+1)\lr{ \sum_{a'\in\subActSet_{K-\ell-1} }\sum_{p=1}^{\ell+1}\sum_{a\in\subActSet_{K-p+1} } \avg\ls{\numMatches_{a,\stableArm_{a'}}(T)} }\\&\hspace{1cm}+ \sum_{p=1}^{\ell+1} g(\ttheta;p)\sum_{a\in\subActSet_{K-p+1}}\sum_{f\in \cup_{j\leq K-\ell-2} \subFirmSet_{j}} \avg\ls{\numMatches_{a,f}} +  \sum_{p=1}^{\ell}f(\ttheta,p)\mathcal{H}_{K-p}\\&\hspace{1cm}+Z(T,\Delta)\sum_{p=1}^{\ell} f(\ttheta,p) |\subActSet_{K-p}|
    \\
    &\underset{(f)}{\leq} \ts_{K-\ell-1} + g(\ttheta;\ell+1)\lr{ \ttheta\mathcal{H}_{K-\ell-1} + \ttheta \mathbb{M}_{K-\ell-1}+\ttheta Z(T,\gap)|\subActSet_{K-\ell-1}| }\\&\hspace{1cm}+ \sum_{p=1}^{\ell+1} g(\ttheta;p)\sum_{a\in\subActSet_{K-p+1}}\sum_{f\in \cup_{j\leq K-\ell-2} \subFirmSet_{j}} \avg\ls{\numMatches_{a,f}} +  \sum_{p=1}^{\ell}f(\ttheta,p)\mathcal{H}_{K-p}\\&\hspace{1cm}+Z(T,\Delta)\sum_{p=1}^{\ell} f(\ttheta,p) |\subActSet_{K-p}|\\
    &\underset{(g)}{=} \ts_{K-\ell-1} + f(\ttheta;\ell+1)\mathbb{M}_{K-\ell-1} + \\&\hspace{1cm}+ \sum_{p=1}^{\ell+1} g(\ttheta;p)\sum_{a\in\subActSet_{K-p+1}}\sum_{f\in \cup_{j\leq K-\ell-2} \subFirmSet_{j}} \avg\ls{\numMatches_{a,f}} +  \sum_{p=1}^{\ell+1}f(\ttheta,p)\mathcal{H}_{K-p}\\&\hspace{1cm}+Z(T,\Delta)\sum_{p=1}^{\ell+1} f(\ttheta,p) |\subActSet_{K-p}|
\end{align*}
where \((a)\) is by induction hypothesis, \((b)\) is by decomposing \(\ts_{K-\ell}\), \((c)\) is by using definition of \(\mathbb{M}_{K-\ell}\), \((d)\) is by rearrangement of terms and using the fact that \(g(\ttheta,\cdot)\)
is increasing, \((e)\) is by rearrangement of terms and using the fact that for any \(f\in \subFirmSet_k\) for some \(k\) there exists \(a'\in \subActSet_k\) such that \(f=\stableArm_{a'}\). Next, \((f)\) is by Lemma \ref{lem: BoundOtherMatch}. Next, \((g)\) is by combining similar terms. This concludes the induction proof. 

We know that \(\ts_1=\mathbb{M}_1=0\) because of \(\alpha-\)reducible structure which ensures that these firms do not have superoptimal firms. Thus in \eqref{eq: InductionTildeS} if take \(\ell=K-1\) then we get 
\begin{align*}
     \ts_{K}&\leq   \sum_{p=1}^{K-1}f(\ttheta,p)\mathcal{H}_{K-p}+Z(T,\Delta)\sum_{p=1}^{K-1} f(\ttheta,p) |\subActSet_{K-p}| \\ 
     &\leq \sum_{p=1}^{K-1}\sum_{j=1}^{p}\ttheta^j\mathcal{H}_{K-p} + Z(T,\Delta)\sum_{p=1}^{K-1} f(\ttheta,p) |\subActSet_{K-p}|\end{align*}
    \begin{align*} 
   \ts_K  &\leq \sum_{j=1}^{K-1}\ttheta^j\sum_{p=j}^{K-1}\mathcal{H}_{K-p} + Z(T,\Delta)\sum_{p=1}^{K-1} f(\ttheta,p) |\subActSet_{K-p}|\\
     &\underset{(a)}{=} \sum_{j=1}^{K-1}\ttheta^jS_{K-j} + Z(T,\Delta)\lr{\sum_{j=1}^{K-1}|{\subActSet}_j|}K\ttheta^{K-1} \\
     &\underset{(b)}{\leq} Z(T,\Delta)\lr{\sum_{j=1}^{K-1}|{\subActSet}_j|}\sum_{j=1}^{K-1}\ttheta^j(K-j)\theta^{K-j-1} + Z(T,\Delta)\lr{\sum_{j=1}^{K-1}|{\subActSet}_j|}K\ttheta^{K-1}
\end{align*}
where \(S_{K-j}\) in (a) is from proof of \textbf{(L4)} in Lemma \ref{lem: MainLemma} and (b) is by \eqref{eq: InductionS}. Define \(\tilde{C}_k=k\ttheta^{k-1}+\sum_{j=1}^{k-1}\ttheta^j(k-j)\theta^{k-j-1}\). Thus we see that 
\begin{align*}
    \ts_{K}\leq |\firmSet|\log(T)\lr{1+\frac{1}{\gap^2}}\lr{\sum_{j=1}^{K-1}|{\subActSet}_j|} \tilde{C}_K
\end{align*}
\end{proof}

\section{Proof of Theorem \ref{thm: UCBMainPaper}}

We now look at the joint regret 
for any \(k\in[K]\). Define \(Z(T,\gap) = |F|\log(T)\lr{1+\frac{1}{\gap^2}}\)
\begin{align*}
   & \sum_{i=1}^{k}\sum_{a\in\subActSet_i}R_{a} \underset{(a)}{=} \bigo\bigg( \sum_{i=1}^{k}\sum_{a\in\subActSet_i}\avg[\numMatches_{a,\subArm_a}(T)] +    \sum_{i=1}^{k}\sum_{a\in\subActSet_i}\sum_{f\in F\backslash \{f^*_a\}} \avg[C_{a,f}(T)] \\&\hspace{1cm}+ \sum_{i=1}^{k}\sum_{a\in\subActSet_i} \avg[\sum_{t=1}^{T}H_{a,f^*_a}(t)]\bigg) \\
    &\underset{(b)}{=} \bigo\lr{  \sum_{i=1}^{k}\sum_{a\in\subActSet_i}\avg[\numMatches_{a,\subArm_a}(T)] +    \sum_{i=1}^{k}\sum_{a\in\subActSet_i} \avg[M_{a,\superArm_a}(T)] + \sum_{i=1}^{k}\sum_{a\in\subActSet_i} \avg[\sum_{t=1}^{T}H_{a,f^*_a}(t)]}\\&\hspace{1cm}+\bigo\lr{|\firmSet|\sum_{i=1}^{k}|\subActSet_i|\log(T)} \\
    &\underset{(c)}{=} \bigo\lr{\sum_{i=1}^{k}\sum_{a\in\subActSet_i} \avg[M_{a,\superArm_a}(T)] + \sum_{i=1}^{k}\sum_{a\in\subActSet_i} \avg[\sum_{t=1}^{T}H_{a,f^*_a}(t)]}+ \bigo(
    \sum_{i=1}^{k}\sum_{a\in\subActSet_i}|\subArm_a|Z(T,\Delta))\\&\hspace{1cm}+\bigo\lr{|F|\sum_{i=1}^{k}|\subActSet_i|\log(T)} \\
    &\underset{(d)}{=} \bigo( \tilde{C}_k\lr{\sum_{p=1}^{k}|\actSet_p|}Z(T,\Delta))+\bigo(\lr{\sum_{p=1}^{k}|\actSet_p|}C_kZ(T,\Delta))+\bigo(
    \sum_{p=1}^{k}\sum_{a\in\subActSet_p}|\subArm_a|Z(T,\Delta))\\&\hspace{1cm}+\bigo\lr{|F|\sum_{p=1}^{k}|\subActSet_p|\log(T)}\\
    &\underset{(e)}{=}\bigo\lr{(C_k+\tilde{C}_k)|\firmSet|\lr{\sum_{p=1}^{k}|\actSet_p|}}\log(T)\lr{1+\frac{1}{\gap^2}}
\end{align*}
where \((a)\) holds due to \textbf{(L1)} in Lemma \ref{lem: MainLemma}, \((b)\) holds due to \textbf{(L3)} in Lemma \ref{lem: MainLemma}, \((c)\) is due to \textbf{(L2)} in Lemma \ref{lem: MainLemma}. Next, \((d)\) is due to \textbf{(L4)-(L5)} in Lemma \ref{lem: MainLemma}. Finally, \((e)\) follows by combining terms.

\section{Technical lemmas}\label{appsec: Proofs}
In this section we present some technical lemmas which are helpful in the proofs in next section.
\begin{lemma}(Lemma 8.2,\cite{lattimore2020bandit})\label{Lem: LemmaLattimore}
Let \(X_1,X_2,\dots,X_T\) be a sequence of independent 1-subgaussian random variable, and \(\hat{\mu}^{(t)}\defas \frac{1}{t}\sum_{s=1}^{t}X_s,\epsilon>0,a>0\) and 
\[
\kappa \defas \sum_{t=1}^{n}\one\lr{\hat{\mu}_t+\sqrt{\frac{2a}{t}}\geq \epsilon}, \quad \kappa' \defas u  + \sum_{t=\ceil{u}}^{T} \one\lr{\hat{\mu}_t+\sqrt{\frac{2a}{t}}\geq \epsilon}
\]
where \(u=\frac{2a}{\epsilon^2}\). Then 
\[
\avg[\kappa] \leq \avg[\kappa']\leq 1+\frac{2}{\epsilon^2}(a+\sqrt{\pi a}+1)
\]
\end{lemma}

\begin{lemma}
Suppose we use the AB subroutine Algorithm \ref{alg:AdaptivePart} with \(\eta \leq 1/50\) then the following two inequalities hold: 
\begin{equation}
\begin{aligned}\label{eq: Implication1}
    &\avg\ls{\sum_{t=1}^{\horizon} \one\lr{\actChoose_{a,f}(t)=1,\actQuery_{a,f}(t)=1,\collisionEvent_{a,f}(t)}}\\&\hspace{1cm}\leq   (1+\varpi)\avg[\numMatches_{a,f}(T)] + \bigo(\log(T)) + \varpi\avg[\numCollide_{a,f}(T)] ,
\end{aligned}
\end{equation}
where \(0<\varpi\leq 32\eta<1\)and 
\begin{equation}
\begin{aligned}\label{eq: Implication2}
    &\avg\ls{\sum_{t=1}^{\horizon}\one\lr{\actChoose_{a,f}(t)=0,\actQuery_{a,f}(t)=1,\collisionEvent_{a,f}^\comp(t)}}\\&\hspace{1cm}\leq \bigo\lr{ \log(T) +\avg\ls{\sum_{t=1}^{T} \one\lr{\collisionEvent_{a,f}(t)}} + \avg[\numPotCollide_{a,f}(T)]}.
\end{aligned}
\end{equation}
\end{lemma}
\begin{proof}
To simplify the presentation of proof, let's define 
\begin{align*}
    \lossAdv_{a,f}(T) \defas \sum_{t=1}^{\horizon} \lr{\one\lr{\actChoose_{a,f}(t)=1,\actQuery_{a,f}(t)=1,\collisionEvent_{a,f}(t)} - \one\lr{\actChoose_{a,f}(t)=1,\actQuery_{a,f}(t)=1,\collisionEvent_{a,f}^\comp(t)} }
\end{align*}

The regret bound for adversarial bandit algorithm from Lemma \ref{lem: AppRegredAdv} under \(\eta\leq 1/50\) implies 
\begin{equation}\label{eq: BothRegret}
    \begin{aligned}
        &\avg\ls{\lossAdv_{a,f}(T)} &\leq \bigo(\log(T)) + \varpi  \avg\ls{\min\lb{\numPotMatches_{a,f}(T),\numPotCollide_{a,f}(T),\numMatches_{a,f}(T)+\numCollide_{a,f}(T)}}   \\ 
        &\avg\ls{\lossAdv_{a,f}(T) - \advLoss_{a,f}(T)} &\leq \bigo(\log(T)) + \varpi \avg\ls{\min\lb{\numPotMatches_{a,f}(T),\numPotCollide_{a,f}(T),\numMatches_{a,f}(T)+\numCollide_{a,f}(T)}}  
    \end{aligned}
\end{equation}
where \(\varpi\leq 32\eta\) and 
\begin{align*}
    \advLoss_{a,f}(T)= \sum_{t=1}^{\horizon} \lr{\one\lr{\actQuery_{a,f}(t)=1,\collisionEvent_{a,f}(t)} - \one\lr{\actQuery_{a,f}(t)=1,\collisionEvent_{a,f}^\comp(t)} }
\end{align*}
which denotes the total loss received by the adversarial bandit subroutine associated with \((a,f)\) in time \(T\) \emph{if} it never take pruning action.  
Therefore, in \eqref{eq: BothRegret} LHS in first inequality is the regret associated with always pruning. While LHS in second inequality is the regret associated with never pruning.

In the following proof we shall analyze each of the equations in \eqref{eq: BothRegret} separately.  
\begin{enumerate}
    \item  \label{enum: Implicaiton1} The first inequality in \eqref{eq: BothRegret} implies 
\begin{align*}
    &\avg\ls{\sum_{t=1}^{\horizon} \lr{\one\lr{\actChoose_{a,f}(t)=1,\actQuery_{a,f}(t)=1,\collisionEvent_{a,f}(t)} - \one\lr{\actChoose_{a,f}(t)=1,\actQuery_{a,f}(t)=1,\collisionEvent_{a,f}^\comp(t)} }} \\ &\quad\quad \quad  \leq  \bigo(\log(T)) + \varpi \lr{\avg[\numMatches_{a,f}(T)+\numCollide_{a,f}(T)]}.
\end{align*}
This in turn leads to 
\begin{align*}
    &\avg\ls{\sum_{t=1}^{\horizon} \lr{\one\lr{\actChoose_{a,f}(t)=1,\actQuery_{a,f}(t)=1,\collisionEvent_{a,f}(t)}}}\\&\leq  \avg\ls{\one\lr{\actChoose_{a,f}(t)=1,\actQuery_{a,f}(t)=1,\collisionEvent_{a,f}^\comp(t)} } +  \bigo(\log(T)) + \frac{1}{2} \lr{\avg[\numMatches_{a,f}(T)+\numCollide_{a,f}(T)]} \\
    &\leq \lr{1+\varpi}\avg[\numMatches_{a,f}(T)] + \bigo(\log(T)) + \varpi\avg[\numCollide_{a,f}(T)]
\end{align*}

\item \label{enum: Implication2} Using the definition of \(\advLoss_{a,f}(T)\) in the second inequality in \eqref{eq: BothRegret} we obtain
\begin{align*}
    &\avg\ls{\sum_{t=1}^{\horizon} \lr{-\one\lr{\actChoose_{a,f}(t)=0,\actQuery_{a,f}(t)=1,\collisionEvent_{a,f}(t)} + \one\lr{\actChoose_{a,f}(t)=0,\actQuery_{a,f}(t)=1,\collisionEvent_{a,f}^\comp(t)} } }\\ &\quad\quad \quad  \leq  \bigo(\log(T) + \avg[\min\{\numPotMatches_{a,f}(T),\numPotCollide_{a,f}(T)\}])
\end{align*}
which implies 
\begin{align*}
    &\avg\ls{\sum_{t=1}^{\horizon}\one\lr{\actChoose_{a,f}(t)=0,\actQuery_{a,f}(t)=1,\collisionEvent_{a,f}^\comp(t)}} \\&\leq \bigo\bigg(\avg\ls{\sum_{t=1}^{\horizon}\one\lr{\actChoose_{a,f}(t)=0,\actQuery_{a,f}(t)=1,\collisionEvent_{a,f}(t)}}  +\bigo(\log(T)) \\&\hspace{1cm}+ \avg[\min\{\numPotMatches_{a,f}(T),\numPotCollide_{a,f}(T)\}]\bigg) \\
    &\leq \bigo\lr{\avg\ls{\sum_{t=1}^{T} \one\lr{\collisionEvent_{a,f}(t)}} + \log(T) + \avg[\min\{\numPotMatches_{a,f}(T),\numPotCollide_{a,f}(T)\}] }
\end{align*}
\end{enumerate}
This concludes the proof. 
\end{proof}

\begin{lemma}[Pruning stable match]\label{lem: BoundingTermA}
For any \(a\in \actSet\),
\begin{align*}
 \underbrace{\avg\ls{\sum_{t=1}^{T}\one\lr{\actChoose_{a,\stableArm_a}(t) = 0,\actQuery_{a,\stableArm_a}(t)=1}}}_{\avg[\text{Term I}]} \leq \bigo\lr{ \avg\ls{\sum_{t=1}^{T}\one\lr{\collisionEvent_{a,\stableArm_a}(t) }}+ \log(T)}
     \end{align*}
\end{lemma}
\begin{proof}
We note that 
\begin{align*}
         \avg[\text{Term I}] &\leq \avg\bigg[\sum_{t=1}^{\horizon}\one\lr{\actChoose_{a,\stableArm_a}(t) = 0,\actQuery_{a,\stableArm_a}(t)=1,\collisionEvent_{a,\stableArm_a}(t)}\\&\hspace{1cm}+\sum_{t=1}^{\horizon}\one\lr{\actChoose_{a,\stableArm_a}(t) = 0,\actQuery_{a,\stableArm_a}(t)=1,\collisionEvent^\comp_{a,\stableArm_a}(t)} \bigg]\\
         &\leq \bigo\lr{ \avg\ls{\sum_{t=1}^{T}\one\lr{\collisionEvent_{a,\stableArm_a}(t) }}+ \bigo(\log(T)) + \avg[\numPotCollide_{a,\stableArm_a}(T)]} \\
         &\leq
         \bigo\lr{ \avg\ls{\sum_{t=1}^{T}\one\lr{\collisionEvent_{a,\stableArm_a}(t) }}+ \bigo(\log(T))}
     \end{align*}
where the first inequality is due to \eqref{eq: Implication2} and the last inequality holds due to Lemma \ref{lem: NumCollisionLemma}.
\end{proof}
\begin{lemma}\label{lem: BoundOtherMatch}
For any \(a\in\actSet\) and \(a'\in \actSet\backslash\{a\}\) we have \begin{align*}
 \sum_{a'\in\actSet}\avg[\numMatches_{a',\stableArm_a}(T)] \leq \bigo\lr{\avg\ls{ \sum_{t=1}^{T} \one\lr{\collisionEvent_{a,\stableArm_a}(t)} }+|\firmSet|Z(T,\Delta)+\avg[\numMatches_{a,\superArm_a}(T)]}
\end{align*}
\end{lemma}
\begin{proof}
For any agent \(a\in \actSet\) we know that at every time step it either gets matched with some firm or gets collided. This implies 
\begin{align}\label{eq: AgentAPerspective}
    \sum_{f'\in\firmSet} \avg[\numCollide_{a,f'}(T)] + \sum_{f'\in\firmSet\backslash\{\stableArm_a\}}\avg[\numMatches_{a,f'}(T)] + \avg[\numMatches_{a,\stableArm_a}(T)] = T. 
\end{align}
Furthermore, in \(T\) steps the firm \(\stableArm_a\) can get matched with some agents or remain unmatched. This implies 
\begin{align}\label{eq: StableFirmPersepective}
    \sum_{a'\in\actSet\backslash\{a\}}\avg[\numMatches_{a',\stableArm_a}(T)] + \avg[\numMatches_{a,\stableArm_a}(T)] \leq T.
\end{align}
Combining \eqref{eq: AgentAPerspective}, \eqref{eq: StableFirmPersepective} and Lemma \ref{lem: NumCollisionLemma} we see that
{
\begin{align*}
    &\sum_{a'\in\actSet}\avg[\numMatches_{a',\stableArm_a}(T)] \leq \sum_{f'\in\firmSet} \avg[\numCollide_{a,f'}(T)] + \sum_{f'\in\firmSet\backslash\{\stableArm_a\}}\avg[\numMatches_{a,f'}(T)] \\ 
    &\leq \bigo\lr{\avg\ls{ \sum_{t=1}^{T} \one\lr{\collisionEvent_{a,\stableArm_a}(t)} }+|\firmSet|\log(T)} + \bigo\lr{ \avg[\numMatches_{a,\subArm_a}(T)] + \avg[\numMatches_{a,\superArm_a}(T)] }. 
\end{align*}
{
Note that from Lemma \ref{lem: numMatchingApp} we have 
\begin{align*}
    \sum_{a'\in\actSet}\avg[\numMatches_{a',\stableArm_a}(T)] &\leq
     \bigo\lr{\avg\ls{ \sum_{t=1}^{T} \one\lr{\collisionEvent_{a,\stableArm_a}(t)} }+|\firmSet|\log(T)+|\subArm_a|Z(T,\Delta)+\avg[\numMatches_{a,\superArm_a}(T)]} \\ 
     &\leq \bigo\lr{\avg\ls{ \sum_{t=1}^{T} \one\lr{\collisionEvent_{a,\stableArm_a}(t)} }+|\firmSet|Z(T,\Delta)+\avg[\numMatches_{a,\superArm_a}(T)]}
\end{align*}

}
This completes the proof.

}
\end{proof}

\
\section{Thompson Sampling based Decentralized Matching Algorithm}
\subsection{Algorithmic Description}
In this section we present a variant of Algorithm \ref{alg:prunedUCBFinal} but with Thompson sampling based stochastic bandit subroutine. For simplicity, we consider the scenario where the noise in \eqref{eq: RewardModel} is sampled from a normal distribution.  To compute the Thompson sampling index each agent \(a\) maintains an empirical average of utility generated from any firm \(f\) till time \(t\) which is \( \mean_{a,f}(t-1)\). At time step \(t\) any agent \(a\in\actSet\) will maintain an index of every firm \(f\in \firmSet\) by sampling it from a normal distribution with mean \(\mean_{a,f}(t-1)\) and variance \(\frac{1}{\sum_{f\in \firmSet}\numMatches_{a,f}}\) (refer line 3 in Algorithm \ref{alg:prunedTSFinal}). 
\begin{algorithm}
\SetAlgoLined
\LinesNumbered
\SetKwInOut{Initialize}{Initialize}
    \Initialize{ $\mean_{a,f}=0,\numMatches_{a,f}=0,\pullProb_{a,f} = 0.5, \auxProb_{a,f} = 0.5, \lossPull_{a,f} = 0, \forall a\in\actSet, f \in \firmSet  $}
    
    \For{$t=1, \ldots , \horizon$}{
    \For{$f \in \firmSet$}{
    Sample $\TSindex_{a,f} \sim \mc{N}\lr{\mean_{a,f},\frac{1}{\totalMatch_a} }$, where \(\totalMatch_a=\sum_{f\in\firmSet}\numMatches_{a,f}\) 
    }
      Set $\TSindex_a$ = \textsf{ArgDescendingSort}($\{\TSindex_{a,f}\}_{f\in\firmSet}$),  
            \(i = 1\)\\
             \While{\(i\leq n\) }
             {
             Set \(f = \TSindex_a^{[i]}\)\\
             Sample \(\pullInstant_{a,f}\sim \textsf{Bernoulli}(\pullProb_{a ,f})\)\\
           \If{\(\pullInstant_{a,f}=0\)}{
           Update \((\auxProb_{a,f},\pullProb_{a,f},\lossPull_{a,f}) \ra  \pullModule(\pullInstant_{a,f},\auxProb_{a,f},\pullProb_{a,f},\lossPull_{a,f},\isMatched_{a})\)
           } \If{\(\pullInstant_{a,f}=1\)}{
            Query firm \(f\) and receive $(\reward_{a},\isMatched_{a})$\\
            Update \(\mean_{a,f} \ra \isMatched_{a}\frac{\mean_{a,f}\numMatches_{a,f}+\reward_{a}}{\numMatches_{a,f}+1} + (1-\isMatched_{a})\mean_{a,f} \) and $\numMatches_{a,f}\ra \numMatches_{a,f}+\isMatched_{a}$,\\
            
            Update \((\auxProb_{a,f},\pullProb_{a,f},\lossPull_{a,f}) \ra  \pullModule(\pullInstant_{a,f},\auxProb_{a,f},\pullProb_{a,f},\lossPull_{a,f},\isMatched_{a})\)\\
            \textsf{break while;}
            }
            \(i \ra i + 1\)
            }
            
            \If{\(i=|\firmSet|+1\)}
            {
           Query a firm \(\TSindex_a^{[1]}\) and receive $(\reward_{a},\isMatched_{a})$\\
            Update \(\mean_{a,f} \ra \isMatched_{a}\frac{\mean_{a,f}\numMatches_{a,f}+\reward_{a}}{\numMatches_{a,f}+1} + (1-\isMatched_{a})\mean_{a,f} \),  $\numMatches_{a,f}\ra \numMatches_{a,f}+\isMatched_{a}$
            }
   }
   \caption{\textsf{Thompson Sampling based Decentralized Matching Algorithm (TS-DMA)}}
    \label{alg:prunedTSFinal}
\end{algorithm}

\subsection{Bounds for Algorithm \ref{alg:prunedTSFinal}}
We first present the regret bound for Algorithm \ref{alg:prunedTSFinal}. 
\begin{theorem}\label{thm: TSSupp} Suppose every agent \(a\in\actSet\) uses Algorithm \ref{alg:prunedTSFinal}.
Then for any \(i\in[ \numMarket]\) :
\begin{align*}
    \sum_{j=1}^{i}\sum_{a\in\subActSet_j}\avg[\regret_a(T)] = \bigo\lr{C_i|\firmSet|{|\actSet|}\lr{\frac{1}{\gap^2} \log\lr{\frac{1}{\gap}}+\frac{ \log(T)}{\gap^2} + \log(T)} }
\end{align*}
where \(\gap=\min_{a,f}\gap_{a,f}\) and \(C_i\) is a constant dependent on market \(\market_i\) and \(C_1<C_2<...<C_{\numMarket}\).
\end{theorem}

The only difference between proof of Theorem \ref{thm: UCBMainPaper} and Theorem \ref{thm: TSSupp} is the bound on expected number of matchings with suboptimal firms (refer \textbf{(L2)} in Lemma \ref{lem: MainLemma}). We now present the analogue of \textbf{(L2)} of Lemma \ref{lem: MainLemma} below.
\begin{lemma}\label{lem: numMatchingAppTS}
 For any \(i\in [\numMarket]\), the expected matches with suboptimal firm satisfies
    \begin{align*}
       &\sum_{j=1}^{i}\sum_{a\in\subActSet_j}\avg[\numMatches_{a,\subArm_a}(T)] \\&= \bigo\lr{  \sum_{j=1}^{i}\sum_{a\in\subActSet_j}\lr{|\subArm_a|\lr{\frac{1}{\gap^2} \log\lr{\frac{1}{\gap}}+\frac{ \log(T)}{\gap^2} + \log(T)} +\avg\ls{\sum_{t=1}^{T}\collisionEvent_{a,\stableArm_a}(t)}}}
    \end{align*}
where \(\gap = \min_{a,f}\gap_{a}(f)\)
\end{lemma}
\begin{proof}
Note that we call an agent \(a\) matches with firm \(f\) at time \(t\) if \(\isMatched_a(t)=1\) and \(\chosenFirm_a(t) = f\). Therefore the total number of matchings between \(a\) and \(f\) till time \(T\) is \(\numMatches_{a,f}(T) = \sum_{t=1}^{T}\one\lr{\isMatched_a(t)=1,\chosenFirm_a(t)=f}\).  Therefore from Lemma \ref{lem: App_SubEvent} and Remark \ref{rem: TSUCB} the following holds for every \(f\in\subArm_a\):
\begin{align*}
&\numMatches_{a,\subArm_a}(\horizon) = \sum_{f\in\subArm_a}\sum_{t=1}^{\horizon} \one\lr{ \isMatched_a(t) = 1 , \chosenFirm_a(t) = f }\\
&\leq \sum_{f\in\subArm_a}\sum_{t=1}^{\horizon}\lr{\one \lr{\isMatched_{a}(t)=1,\chosenFirm_a(t)=f,\TSindex_{a,f}(t)\geq \TSindex_{a,\stableArm_a}(t)}  +   \one\lr{\actChoose_{a,f}(t)=1,\actChoose_{a,\stableArm_a}=0}}\\
&\leq \sum_{f\in\subArm_a}\sum_{t=1}^{\horizon}\one \lr{\isMatched_{a}(t)=1,\chosenFirm_a(t)=f,\TSindex_{a,f}(t)\geq \TSindex_{a,\stableArm_a}(t)} \\&\hspace{3cm} +  \sum_{t=1}^{\horizon}\sum_{f\in\subArm_a} \one\lr{\actChoose_{a,f}(t)=1,\actChoose_{a,\stableArm_a}=0}
\\
&\leq\sum_{f\in\subArm_a}\underbrace{\sum_{t=1}^{\horizon}\one \lr{\isMatched_{a}(t)=1,\chosenFirm_a(t)=f,\TSindex_{a,f}(t)\geq \TSindex_{a,\stableArm_a}(t)}}_{\text{Term A}}  +  \underbrace{\sum_{t=1}^{\horizon} \one\lr{\actChoose_{a,\stableArm_a}=0}}_{\text{Term B}}
\end{align*}

Let's first analyze Term \(A\). Define \(\filtration_{t-1} = \{\{\chosenFirm_a(\tau),Y_a(\tau), \reward_a(\tau)\}_{\tau=1}^{t-1}\}_{a\in \actSet}\). We first observe that 
\begin{equation}
\begin{aligned}\label{eq: FactorTS}
    &\one\lr{\isMatched_a(t)=1,\actChoose_{a,f}(t)=1,\actQuery_{a,f}(t)=1, \TSindex_{a,\stableArm_a} \leq \TSindex_{a,f}(t)} \\&\quad = \underbrace{\one\lr{\isMatched_a(t)=1,\actChoose_{a,f}(t)=1,\actQuery_{a,f}(t)=1, \TSindex_{a,\stableArm_a} \leq \TSindex_{a,f}(t),\TSindex_{a,f}(t) < \mean_{a,\stableArm_a}-\epsilon}}_{\text{Term C}} \\ 
    &\quad + \underbrace{\one\lr{\isMatched_a(t)=1,\actChoose_{a,f}(t)=1,\actQuery_{a,f}(t)=1, \TSindex_{a,\stableArm_a} \leq \TSindex_{a,f}(t),\TSindex_{a,f}(t) \geq  \mean_{a,\stableArm_a}-\epsilon}}_{\text{Term D}}
\end{aligned}
\end{equation}

We first provide a bound on Term C. Prior to that let's define some notations. Let's define \(\goodProb_{a,f}^{(s)}(\epsilon) = 1-\cdfUpdate_{a,f}^{(s)}(\mean_{a,\stableArm_a} - \epsilon) \). Furthermore,
conditioned on the event that atleast one arm is pulled, for any agent \(a\) let's define \(\mc{P}_a(t)\) to be the set of arms that are pruned before one is chosen to be played at time \(t\). Moreover let \(\actQueryN_{a,f}(t)\) be a random variable such that \(\actQueryN_{a,f}(t)=1\) iff \(f\) is the firm with maximum index value in all of the non-pruned arms at time \(t\). That is,  
\(
\actQueryN_{a,f}(t) = \one\lr{f\in \argmax_{f'\in\firmSet\backslash \{\mc{P}(t)\cup \{\stableArm_a\}\}} \TSindex_{a,f'}(t)}.
\) Using this the following holds:
\begin{align}\label{eq: TS_Eq1}
    &\avg[\text{Term C}] = \avg[\avg[\text{Term C}|\filtration_{t-1}]] \notag \\ 
    &=\avg[\Pr{\isMatched_a(t)=1,\actChoose_{a,f}(t)=1,\actQuery_{a,f}(t)=1, \TSindex_{a,\stableArm_a} \leq \TSindex_{a,f}(t),\TSindex_{a,f}(t) < \mean_{a,\stableArm_a}-\epsilon | \filtration_{t-1}}]\notag \\  & \leq \avg\ls{\Pr{\TSindex_{a,\stableArm_a} < \mean_{a,\stableArm_a}-\epsilon | \filtration_{t-1}} \Pr{\isMatched_a(t)=1,\actQueryN_{a,f}(t)=1,\TSindex_{a,f}(t)<\mean_{a,\stableArm_a}-\epsilon | \filtration_{t-1}}}
\end{align}

Moreover note that 
\begin{align}\label{eq: TS_Eq2}
    &\Pr{\isMatched_a(t)=1,\actQuery_{a,\stableArm_a}(t)=1,\TSindex_{a,f}(t)(t)<\mean_{a,\stableArm_a}-\epsilon|\filtration_{t-1}} \notag \\ &\geq \Pr{\isMatched_a(t)=1,\actQueryN_{a,f}(t)=1, \TSindex_{a,f}(t)(t)<\mean_{a,\stableArm_a}-\epsilon, \TSindex_{a,\stableArm_a}(t) > \mean_{a,\stableArm}-\epsilon|\filtration_{t-1}} \notag \\ 
    &=  \Pr{ \TSindex_{a,\stableArm_a}(t) > \mean_{a,\stableArm_a}(t-1) -\epsilon|\filtration_{t-1}}\Pr{\isMatched_a(t)=1,\actQueryN_{a,f}(t)=1, \TSindex_{a,f}(t)(t)<\mean_{a,\stableArm_a}-\epsilon | \filtration_{t-1}  } 
\end{align}

Using \eqref{eq: TS_Eq2} in \eqref{eq: TS_Eq1} we obtain the following
 \begin{align*}
     &\avg[\text{Term C}] = 
     \avg\bigg[\frac{\Pr{\TSindex_{a,\stableArm_a} < \mean_{a,\stableArm_a}-\epsilon | \filtration_{t-1}}}{ \Pr{ \TSindex_{a,\stableArm_a}(t) > \mean_{a,\stableArm_a}(t-1) -\epsilon|\filtration_{t-1}}}\cdot \\ &\hspace{3cm}\Pr{\isMatched_a(t)=1,\actQuery_{a,\stableArm_a}(t)=1,\TSindex_{a,f}(t)(t)<\mean_{a,\stableArm_a}-\epsilon|\filtration_{t-1}}\bigg]
     \\& = \avg\ls{\frac{1-\goodProb_{a,\stableArm_a}^{(\numMatches_{a,\stableArm_a}(t-1))}(\epsilon)}{\goodProb_{a,\stableArm_a}^{(\numMatches_{a,\stableArm_a}(t-1))}(\epsilon)} \Pr{\isMatched_a(t)=1,\actQuery_{a,\stableArm_a}(t)=1,\TSindex_{a,f}(t)(t)<\mean_{a,\stableArm_a}-\epsilon|\filtration_{t-1}}} \\ 
     &\leq \avg\ls{\frac{1-\goodProb_{a,\stableArm_a}^{(\numMatches_{a,\stableArm_a}(t-1))}(\epsilon)}{\goodProb_{a,\stableArm_a}^{(\numMatches_{a,\stableArm_a}(t-1))}(\epsilon)}\Pr{\isMatched_a(t)=1,\actQuery_{a,\stableArm_a}(t)=1|\filtration_{t-1}}}
 \end{align*}
Further evaluating the expectation of Term C we have:
\begin{align*}
    \avg[\text{Term C}] &= \sum_{t=1}^{T}\avg\ls{\frac{1-\goodProb_{a,\stableArm_a}^{(\numMatches_{a,\stableArm_a}(t-1))}(\epsilon)}{\goodProb_{a,\stableArm_a}^{(\numMatches_{a,\stableArm_a}(t-1))}(\epsilon)} \one\lr{\actQuery_{a,\stableArm_a}(t)=1, \actChoose_{a,\stableArm_a}(t)=1, \isMatched_a(t) = 1 }} \\
    &= \sum_{t=1}^{T}\sum_{s=1}^{t} \avg\ls{ \frac{1-\goodProb_{a,\stableArm_a}^{(s)}(\epsilon)}{\goodProb_{a,\stableArm_a}^{(s)}(\epsilon)} \one\lr{\actQuery_{a,\stableArm_a}(t)=1, \actChoose_{a,\stableArm_a}(t)=1, \isMatched_a(t) = 1, \numMatches_{a,\stableArm_a}(t-1) =s }  } \\ 
    &\leq  \avg\ls{\sum_{s=1}^{T}\frac{1-\goodProb_{a,\stableArm_a}^{(s)}(\epsilon)}{\goodProb_{a,\stableArm_a}^{(s)}(\epsilon)} \sum_{t=s+1}^{T} \one\lr{\numMatches_{a,f}(t-1)=s, \numMatches_{a,f}(t)= s+1}} \\ 
     &\leq \sum_{s=0}^{\infty}\frac{1-\goodProb_{a,\stableArm_a}^{(s)}(\epsilon)}{\goodProb_{a,\stableArm_a}^{(s)}(\epsilon)}\leq{ \frac{1}{\epsilon^2} \log(\frac{1}{\epsilon})}
\end{align*}
where the last inequality is due to \cite{lattimore2020bandit}.
Now let's look at Term D. Let's set of time indices when \(\mc{J}_{a,f} = \{t: \goodProb_{a,f}^{(\numMatches_{a,f}(t-1))}(\epsilon)> 1/T\}\). 
\begin{align*}
     \avg[\text{Term D}] &=\sum_{t=1}^{T}\avg\ls{\one\lr{\isMatched_a(t)=1,\actChoose_{a,f}(t)=1,\actQuery_{a,f}(t)=1, \TSindex_{a,\stableArm_a} \leq \TSindex_{a,f}(t),\TSindex_{a,f}(t) \geq  \mean_{a,\stableArm_a}-\epsilon}} \\ 
     &\leq \underbrace{\sum_{t\in\mc{J}_{a,f}} \avg\ls{ \one\lr{ \isMatched_a(t) = 1, \actChoose_{a,f}(t)= 1 }}}_{\text{Term E}}  + \underbrace{\sum_{t\not\in \mc{J}_{a,f}} \avg\ls{\one\lr{\TSindex_{a,f}(t) \geq  \mean_{a,\stableArm_a}-\epsilon }}}_{\text{Term F}} 
\end{align*} 
Let's first analyze the Term E above. Note that 
\begin{align*}
    &\sum_{t\in \mc{J}_{a,f}} \one\lr{ \isMatched_a(t) = 1, \actChoose_{a,f}(t)= 1 } \\&\leq  \sum_{t=1}^{T}\sum_{s=1}^{t-1} \one\lr{ \isMatched_a(t) = 1, \actChoose_{a,f}(t)= 1, \goodProb_{a,f}^{s}(\epsilon) > \frac{1}{T}, \numMatches_{a,f}(t-1)=s, \numMatches_{a,f}(t) = s+1} \\ 
    &= \sum_{s=0}^{T-1}\one\lr{\goodProb_{a,f}^{(s)}(\epsilon) > \frac{1}{T}}\sum_{t=s+1}^{T}  \one\lr{\numMatches_{a,f}(t-1)=s, \numMatches_{a,f}(t) = s+1 } \\ 
    &= \sum_{s=0}^{T-1}\one\lr{\goodProb_{a,f}^{(s)}(\epsilon) > \frac{1}{T}} \leq \bigo\lr{\frac{ \log(T)}{(\Delta_{a,f}-\epsilon)^2} + \log(T)}
\end{align*}
where the last property is a property of concentration of normal distribution and is standard in frequentist Thompson sampling analysis. For reader's reference we point to the book \cite {lattimore2020bandit}.
Next, we bound Term F below: 
\begin{align*}
    &\sum_{t\not\in \mc{J}_{a,f}} \avg\ls{\one\lr{\TSindex_{a,f}(t) \geq  \mean_{a,\stableArm_a}-\epsilon }} = \sum_{t=1}^{T} \avg\ls{\one\lr{\TSindex_{a,f}(t) \geq  \mean_{a,\stableArm_a}-\epsilon, \goodProb_{a,f}^{(\numMatches_{a,f}(t-1))}(\epsilon)\leq \frac{1}{T} }} \\ 
    &=\sum_{t=1}^{T}\avg\ls{\avg\ls{\one\lr{\TSindex_{a,f}(t) \geq  \mean_{a,\stableArm_a}-\epsilon, \goodProb_{a,f}^{(\numMatches_{a,f}(t-1))}(\epsilon)\leq \frac{1}{T} } }|\mc{F}_{t-1}} \\ 
    &= \sum_{t=1}^{T}\avg\ls{\goodProb_{a,f}^{(\numMatches_{a,f}(t-1))}(\epsilon) \one\lr{\goodProb_{a,f}^{(\numMatches_{a,f}(t-1))}(\epsilon)<\frac{1}{T}}} \\ 
    &\leq 1
\end{align*}

Combining the bounds on Term C, Term E and Term F and choosing \(\epsilon = \frac{\gap}{2}\) we have 
\begin{align*}
    \sum_{f\in\subArm_a}\avg[\numMatches_{a,f}(T)] &\leq |\subArm_a|\bigo\lr{ \frac{1}{\gap^2} \log\lr{\frac{1}{\gap}}+\frac{ \log(T)}{\gap^2} + \log(T) }\\&\hspace{1cm}+ \avg\ls{\sum_{t=1}^{T} \one\lr{ \actQuery_{a,\stableArm_a}(t)=1, \actChoose_{a,\stableArm_a}(t)=0 } } \\ 
    &\leq |\subArm_a|\bigo\lr{ \frac{1}{\gap^2} \log\lr{\frac{1}{\gap}}+\frac{ \log(T)}{\gap^2} + \log(T)} + \bigo\lr{\avg\ls{\sum_{t=1}^{T}\one\lr{\collisionEvent_{a,\stableArm_a}(t) }}}
\end{align*}
where the second inequality is due to Lemma \ref{lem: BoundingTermA}.  This concludes the proof. 
\end{proof}

\section{Table of Notations}
We have accumulated all the main notations used in the paper in form of table below
\begin{table}[h]
    \centering
    \begin{tabular}{|c|l|}
        \hline
        \textbf{Notation} & \textbf{Description} \\
        \hline
        \hline 
        \(\actSet\) & Set of agents \\
        \(\firmSet\) & Set of firms/arms \\
        \(\market\) & Union of agents and firms\\
        \(\utilityAgent_a(f)\) & Utility for agent \(a\) when matched with firm \(f\) \\
        \(\utilityFirm_f(a)\) & Utility for firm \(f\) when matched with agent \(a\) \\
        \(\chosenFirm_a(t)\) & Firm chosen by agent \(a\) at time \(t\)\\
       \(\stableArm_a\) & Stable match of agent \(a\)
       \\
       \(\superArm_a\) & Set of super-optimal firms for agent \(a\)
        \\
        \(\subArm_a\) & Set of sub-optimal firms for agent \(a\)
        \\ \(\numMarket\) & Number of markets formed by decomposition
        as stated in Remark \ref{rem: MarketDecomp}
        \\ \(\subActSet_i\) & Agents forming fixed pairs after \(i-1\) rounds of elimination (Remark \ref{rem: MarketDecomp})
        \\ \(\subFirmSet_i\) & Firms forming fixed pairs after \(i-1\) rounds of elimination (Remark \ref{rem: MarketDecomp})
        \\
        \(\reward_{a,f}\) & Noisy reward that agent \(a\) receives on getting matched with firm \(f\)
        \\
        \(\agentsPull_f\) & Set of agents that pull firm \(f\)\\
        \(\numMatches_{a,f}(T)\) & Number of times agent \(a\) has successfully matched with firm \(f\) till time \(T\) \\
        \(\numCollide_{a,f}(T)\) & Number of times agent \(a\) has collided on firm \(f\) till time \(T\) \\
        \(\pullProb_{a,f}(t)\) & Probability that agent \(a\) will pull firm \(f\) at time \(t\)
        \\
        \(\pullInstant_{a,f}(t)\) & An indicator if agent \(a\) has pulled arm \(f\) at time \(t\)\\
        \(\isMatched_a(t)\) & An indicator if agent \(a\) got successfully matched at time \(t\) \\
        \(\mean_{a,f}(t)\) & Empirical mean of utility derived by agent \(a\) on matching with \(f\)\\
      \(\UCB_{a,f}(t)\) & UCB estimate of reward from firm \(f\) to agent \(a\) at time \(t\)\\
      \(\TSindex_{a,f}(t)\) & Thompson Sampling index of reward from firm \(f\) to agent \(a\) at time \(t\)
       \\
        \( \actChoose_{a,f}(t)\) & An indicator if agent \(a\) pulled firm \(f\) at time \(t\) \\ 
     \(\actQuery_{a,f}(t)\) & An indicator if all the firms with higher index than \(f\) got pruned at time \(t\) \\ 
        
        \(\numSelect_{a,f}(T)\) & Time steps during which \(\actQuery_{a,f}(t)=1\)
        \\
        \(\gap_{a,f}\) & \(\utilityAgent_a(\stableArm_a)-\utilityAgent_a(f)\)
        \\
        \hline
    \end{tabular}
    \caption{Table of notations}
    \label{tab:NotationsTable}
\end{table}

\end{document}